\documentclass{article}

\usepackage[utf8]{inputenc} 
\usepackage[T1]{fontenc}    
\usepackage{hyperref}       
\usepackage{url}            
\usepackage{booktabs}       
\usepackage{amsfonts}       
\usepackage{nicefrac}       
\usepackage{microtype}      
\usepackage{xcolor}         
\usepackage{pifont}
\usepackage{amsmath}
\usepackage{amssymb}
\usepackage{mathtools}
\usepackage{amsthm}
\usepackage[capitalize,noabbrev]{cleveref}

\usepackage{algorithm}
\usepackage{algorithmic}

\usepackage{graphicx}
\usepackage{subcaption}
\usepackage{adjustbox}

\theoremstyle{plain}
\newtheorem{theorem}{Theorem}[section]

\newtheorem{lemma}[theorem]{Lemma}

\theoremstyle{definition}
\newtheorem{definition}[theorem]{Definition}
\newtheorem{assumption}[theorem]{Assumption}
\theoremstyle{remark}

\usepackage[textsize=tiny]{todonotes}

\newcommand{\A}{\mathcal{A}}
\newcommand{\D}{\mathcal{D}}
\newcommand{\B}{\mathcal{B}}
\newcommand{\LL}{\mathcal{L}}
\newcommand{\E}{\mathbb{E}}
\newcommand{\argmax}{\arg\max}
\newcommand{\ignore}[1]{}

\usepackage[skip=1.5pt plus1.2pt, indent=1pt]{parskip}
\usepackage{enumitem}
\setitemize{noitemsep,topsep=0pt,parsep=0pt,partopsep=0pt,leftmargin=*}

\usepackage[textsize=tiny]{todonotes}

\usepackage[accepted]{icml2023}

\icmltitlerunning{ICML 2023 Workshop on The 2nd New Frontiers
In Adversarial Machine Learning}
\begin{document}

\twocolumn[
\icmltitle{A First Order Meta Stackelberg Method for Robust Federated Learning}



\icmlsetsymbol{equal}{*}

\begin{icmlauthorlist}
\icmlauthor{Yunian Pan}{equal,nyu}
\icmlauthor{Tao Li}{equal,nyu}
\icmlauthor{Henger Li}{tulane}
\icmlauthor{Tianyi Xu}{tulane}
\icmlauthor{Zizhan Zheng}{tulane}
\icmlauthor{Quanyan Zhu}{nyu}
\end{icmlauthorlist}

\icmlaffiliation{nyu}{Department of Electrical and Computer Engineering, New York, NY, USA}
\icmlaffiliation{tulane}{Department of Computer Science, Tulane University, New Orleans, LA, USA}

\icmlcorrespondingauthor{Tao Li}{taoli@nyu.edu}

\icmlkeywords{Adversarial Federated Learning, Meta Learning, Meta Equilibrium, Bayesian Stackelberg Markov Game, ICML}

 \vskip 0.3in
]



\printAffiliationsAndNotice{\icmlEqualContribution} 

\begin{abstract}
\vspace{10pt}
Previous research has shown that federated learning (FL) systems are exposed to an array of security risks. Despite the proposal of several defensive strategies, they tend to be non-adaptive and specific to certain types of attacks, rendering them ineffective against unpredictable or adaptive threats. This work models adversarial federated learning as a Bayesian Stackelberg Markov game (BSMG) to capture the defender's incomplete information of various attack types. We propose meta-Stackelberg learning (meta-SL), a provably efficient meta-learning algorithm, to solve the equilibrium strategy in BSMG, leading to an adaptable FL defense. We demonstrate that meta-SL converges to the first-order $\varepsilon$-equilibrium point in $O(\varepsilon^{-2})$ gradient iterations, with $O(\varepsilon^{-4})$ samples needed per iteration, matching the state of the art. Empirical evidence indicates that our meta-Stackelberg framework performs exceptionally well against potent model poisoning and backdoor attacks of an uncertain nature.
\end{abstract}

\section{Introduction}
Federated learning (FL) provides a way for several devices possessing private data to collaboratively train a learning model without the need to share their local data~\cite{mcmahan2017communication}. Nonetheless, FL systems remain susceptible to antagonistic attacks, including untargeted model poisoning and specific backdoor attacks. To counter these vulnerabilities, a range of robust aggregation techniques like Krum~\cite{blanchard2017machine}, coordinate-wise median~\cite{yin2018byzantine}, trimmed mean~\cite{yin2018byzantine}, and FLTrust~\cite{cao2020fltrust} have been suggested for defense against non-specific attacks. Furthermore, different post-training protective measures like Neuron Clipping~\cite{wang2022universal} and Pruning~\cite{wu2020mitigating} have been recently introduced to reduce the impact of backdoor attacks.

Existing defenses are typically built to resist specific attack types and attacks that do not evolve in response to defensive measures. In this work, we introduce a meta-Stackelberg game (meta-SG) framework that delivers robust defensive performance, even against adaptive attacks such as the reinforcement learning (RL)-based attack~\cite{li2022learning}, which current state-of-the-art defenses struggle to address, or an amalgamation of different attack types like the simultaneous occurrence of model poisoning and backdoor attacks (see Section~\ref{sec:exp}).

\begin{figure}
    \centering
    \includegraphics[width=0.45\textwidth]{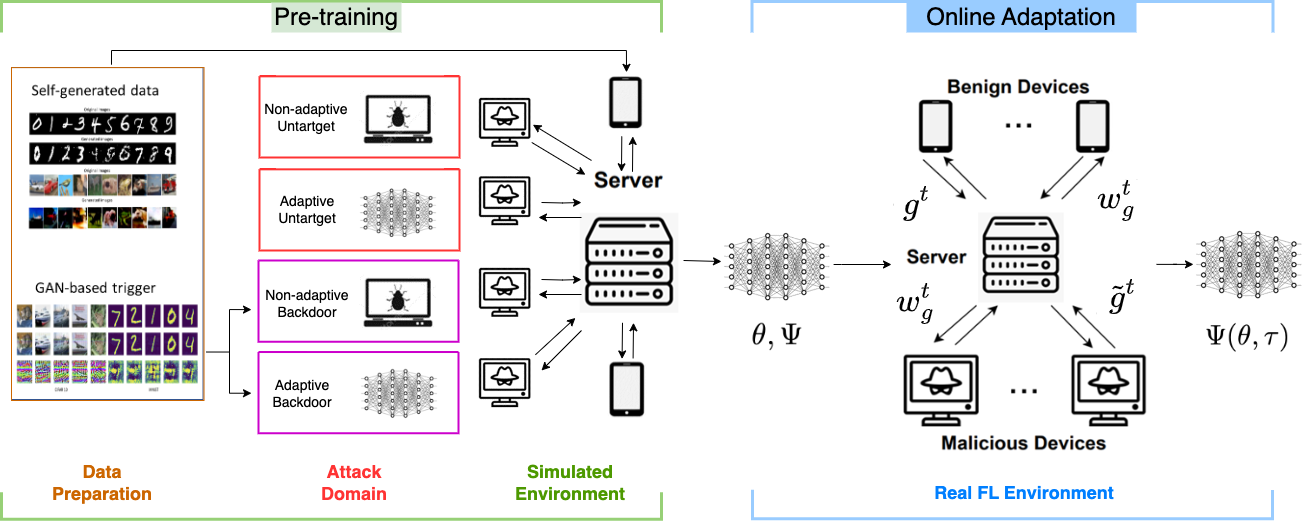}
    \caption{A schematic illustration of the meta-Stagberg game framework.  In the pertaining stage, a simulated environment is constructed using generated data and {a set of attacks sampled from the attack domain consisting of various attack strategies}. The defender utilizes meta-Stackelberg learning (\Cref{algo:meta-sl}) to obtain the meta policy $\theta$ and the adaptation $\Psi$ in \eqref{eq:meta-se}. Then, in the online execution, the defender can adapt its defense to $\Psi(\theta,\tau)$ using received feedback $\tau$ in the presence of unknown attacks.}
    \label{fig:fl-game}
    \vspace{-0.6cm}
\end{figure}
Our meta-SG defense framework is established on several key observations. Firstly, the issue of robust FL in the face of a non-adaptive attack can be perceived as a Markov decision process (MDP), where the state represents model updates from selected devices, and the action refers to the gradient for updating the global model. Moreover, when the attack is known beforehand, the defender can employ the limited amount of local data at the server and publicly accessible data to construct an (approximate) MDP model and determine a robust defense policy by pre-training prior to the commencement of FL training. Secondly, for situations where the attack is adaptive but with specific parameters, we consider a Markov game between the attacker and the defender and establish a robust defense by solving the Stackelberg equilibrium of the game, wherein the defender is the leader and the attacker the follower. This approach applies to both single and multiple concurrent attacks and may yield an (almost) optimal defense. Thirdly, in more realistic scenarios where attacks are unknown or uncertain, this situation can be treated as a Bayesian Stackelberg Markov game (BSMG), offering a comprehensive model for adversarial FL. {Uncertain attacks refer to those involved in the pre-training stage but undisclosed in the online FL process, leaving the defender unsure about their existence. On the other hand, unknown attacks point to those excluded in the pre-training, of which the defender is unaware. }Nonetheless, the standard solution concept for BSMG, the Bayesian Stackelberg equilibrium, aims at the expected case and does not adapt to the actual attack.


In this study, we introduce a novel solution concept, the meta-Stackelberg equilibrium (meta-SE), for BSMG as a systematic approach to creating resilient and adaptive defenses for federated learning. By merging meta-learning with Stackelberg reasoning, meta-SE provides a computationally efficient method to handle information asymmetry in adversarial FL and facilitates strategic adaptation during online execution amidst multiple (adaptive) attackers. Prior to training an FL model, a meta-policy is trained by solving the BSMG using experiences sampled from a set of potential attacks. During FL training, when confronted with an actual attacker, the meta-policy rapidly adapts using a relatively small batch of samples gathered in real-time. Importantly, our proposed Meta-SG framework only requires a rough estimate of potential attacks during meta-training due to the generalization capability offered by meta-learning.

To solve the BSMG in the pre-training stage, we develop a meta-Stackelberg learning (meta-SL) algorithm, based on the concept of debiased meta-reinforcement learning \cite{fallah2021convergence}. Meta-SL is proven to converge to the first-order $\varepsilon$-approximate meta-SE in $O(\varepsilon^{-2})$ iterations, and the corresponding sample complexity per iteration is $O(\varepsilon^{-4})$. Such algorithmic complexity aligns with the latest complexity results in nonconvex bi-level stochastic optimization \cite{ji2021bilevel}. Due to the conflicting interests between the defender and the attacker in FL, the ensuing BSMG is strictly competitive, which can be seen as a generalization of zero-sum. Therefore, meta-SL does not require second-order derivatives of the attacker's value function (the low-level problem), even though the Hessian of the defender's value function remains due to the meta adaptation. Inspired by Reptile \cite{nichol2018first}, a first-order meta-learning algorithm, we propose a fully first-order pre-training algorithm, referred to as Reptile Meta-SL, as a substitute for meta SL in our experiments. Reptile Meta-SL uses only the first-order stochastic gradients of the attacker's and defender's objective functions to solve for the approximate equilibrium. As evidenced by numerical results in \Cref{sec:exp} and Appendix, it is effective in managing adaptive and/or uncertain (or unknown) attacks.

\textbf{Our contributions} can be summarized as follows, with the discussion of related work relocated to the Appendix due to space constraints:
\begin{itemize}
\item We tackle vital security issues in federated learning in the face of multiple adaptive (non-adaptive) attackers of uncertain or unknown types.
\item We formulate a Bayesian Stackelberg game model (\Cref{subsec:bsmg}) to encapsulate the information asymmetry in adversarial FL under uncertain or unknown adaptive attacks.
\item To provide the defender with strategic adaptability, we introduce a new equilibrium concept, the meta-Stackelberg equilibrium (\Cref{def:meta-se}). Here, the defender (the leader) commits to a meta policy and an adaptation strategy, leading to a data-driven method to handle information asymmetry.
\item To learn the meta equilibrium defense during the pre-training phase, we develop meta-Stackelberg learning (\Cref{algo:meta-sl}), an efficient first-order meta RL algorithm. This algorithm provably converges to $\varepsilon$-approximate equilibrium in $O(\varepsilon^{-2})$ gradient steps with $O(\varepsilon^{-4})$ samples per iteration, 
matching the state-of-the-art efficiency in stochastic bilevel optimization.
\item We carry out comprehensive experiments in real-world scenarios to demonstrate the outstanding performance of our proposed method.
\end{itemize}

\section{Model Formulation}

\subsection{Federated Learning and Threat Model}
\textbf{FL objective.} Consider a learning system that includes one server and $n$ clients, each client possesses its own private dataset $D_i={(x_i^j,y_i^j)_{j=1}^{|D_i|}}$ and $|D_i|$ signifies the size of the dataset for the $i$-th client.
Let $U=\{D_1, D_2, \dots, D_n\}$ represent the compilation of all client datasets. 
The objective of federated learning is defined as identifying a model $w$ that minimizes the average loss across all the devices: $\min_w F(w, U):=\frac{1}{n}\sum_{i=1}^{n}f(w, D_i)$, where $f(w, D_i):=\frac{1}{|D_i|}\sum_{j=1}^{|D_i|}\ell(w,(x_i^j,y_i^j))$ is the local empirical loss with $\ell(\cdot,\cdot)$ being the loss function.

\textbf{Attack objective.} 
We consider two major categories of attacks, namely, backdoor attacks and untargeted model poisoning attacks. 
Our framework can be extended to other attack scenarios.
For simplicity, assume that the first $M_1$ malicious clients carry out the backdoor attack and the following $M_2$ malicious clients undertake the poisoning attack.
The model poisoning attack aims to maximize the average model loss, i.e., $\max_w F(w)$; 
the backdoor attack aims to preserve decent performance on clean test inputs (``main task'') while causing misclassification of poisoned test inputs to one or more target labels (``backdoor task''). 
Each malicious client in the backdoor attack produces a poisoned data set $D'_{i\leq M_1}$, obtained by altering a subset of data samples $(x_i^j,y_i^j) \in D_i$ to $(\hat{x}^{j}_i,c^*)$, where $\hat{x}^{j}_i$ is the tainted sample with a backdoor trigger inserted, and $c^* \neq y^j_i, c^*\in C$ is the targeted label. 
Let $U'=\{D'_1, D'_2, \dots, D'_{M_1}\}$ denote the compilation of poisoned datasets. The objective function in the backdoor attack is defined as: $\min_w F'(w)=\lambda F(w, U)+(1-\lambda) F(w,U')$, where $\lambda\in [0,1]$ serves to balance between the main task and the backdoor task.

\textbf{FL process.} 
The federated learning process works in an adversarial setting as follows. 
At each round $t$ out of $H$ FL rounds, the server randomly selects a subset of clients $\mathcal{S}^t$ and sends them the most recent global model $w_g^t$. Every benign client in $\mathcal{S}^t$ updates the model using their local data via one or more iterations of stochastic gradient descent and returns the model update $g^t$ to the server. Conversely, an adversary in $\mathcal{S}^t$ creates a malicious model update $\widetilde{g}^t$ clandestinely and sends it back. The server then collects the set of model updates $\{{\widetilde{g}_i^t}\cup {\widetilde{g}_j^t}\cup {g_k^t}\}_{i,j,k\in \mathcal{S}^t, i\in [M_1], j\in [M_2], k\notin [M_1]\cup [M_2]}$, utilizing an aggregation rule $Aggr$ to combine them and updates the global model $w_g^{t+1}=w_g^t-Aggr({\widetilde{g}_i^t}\cup {\widetilde{g}_j^t}\cup {g_k^t})$, which is then sent to clients in round $t+1$. At the final round $T$, the server applies a post-training defense $h(\cdot)$ on the global model to generate the final global model $\widehat{w}_g^T=h(w_g^T)$.

\textbf{Attacker type and behavior.} We anticipate multiple types of attacks occurring simultaneously, emanating from various categories. For clarity, we consider a single mastermind attacker present within the FL system who controls a group of malicious clients employing diverse attack strategies, which may be either non-adaptive or adaptive. Non-adaptive attacks involve a fixed attack strategy that solves a short-sighted optimization problem, disregarding the defense mechanism implemented by the server (i.e., the 
robust aggregation rule and the post-training defense). Such attacks include explicit boosting (EB)~\cite{bhagoji2019analyzing}, inner product manipulation (IPM)~\cite{xie2020fall}, and local model poisoning attack (LMP)~\cite{fang2020local}), federated backdoor attack (BFL)~\citep{bagdasaryan2020backdoor}, distributed backdoor attack (DBA)~\citep{xie2019dba}, projected gradient decent (PGD) backdoor attack~\citep{wang2020attack}. On the other hand, an adaptive attack, such as the RL-based model poisoning attack~\cite{li2022learning} and RL-based backdoor attack~\cite{li2023learning}, designs model updates by simulating the server's reactions to optimize a long-term objective. {These adaptive and non-adaptive untargeted/backdoor attack methods mentioned above constitute the attack domain in our adversarial FL setting.} One significant hurdle in addressing adversarial attacks is the \textit{information asymmetry}~\cite{tao_info}. This is when the server (i.e., the defender) lacks knowledge of the behavior and identities of malicious clients in a realistic black-box scenario. We denote the collective attack configuration of malicious clients as the type of the mastermind attacker, detailing $M_1, M_2$, attack behaviors (adaptive or not), and other required parameters of the attack method.

\subsection{Bayesian Stackelberg Markov Game Model}
\label{subsec:bsmg}
In this study, we propose a comprehensive framework for robust defense against potent unknown or uncertain attacks. 
The central principle is to construct RL-based defenses by simulating unknown attack behavior using RL-based attacks. 
As demonstrated in prior research~\cite{li2022learning,li2023learning}, RL-based attacks serve as a robust baseline for both model poisoning and backdoor attacks. Therefore, a defense that is resilient to RL-based attacks could potentially safeguard the system against other (less potent) attacks. To manage the high-dimensional state and action spaces, we 
integrate a 
set 
of lightweight defenses in the training stage and post-training stage. The groundbreaking element of our approach is the use of RL to optimize these defenses, moving away from the conventional fixed and manually-tuned hyperparameters. This approach requires a Bayesian Stackelberg Markov game formulation, encapsulated in the tuple $G = (\mathcal{P}, Q, S, O, A, \mathcal{T}, r, \gamma)$, where $\gamma \in(0,1)$ is the reward discounting factor:  



\begin{itemize}
    \item The player set $\mathcal{P}=\{\D, \mathcal{A}\}$ contains $\mathcal{D}$ as the leader (defender), and $\mathcal{A}$ as the follower (attacker) who controls multiple malicious clients.
    \item $Q(\cdot): \Xi \to [0,1]$ denotes the probability distribution over the attacker's private types. $\Xi := \{\xi_i\}_{i=1}^{|\Xi|}$ where $\xi_i$ denotes $i$-th type attacks.
    \item  $O$ is the observation space; the observation for the server (i.e., defender) at round $t$ is $w^{t}_g$ (the server does not have access to the client's identities); the observation for the attacker at round $t$ is $s^t:=(w^{t}_g,\mathbf{I}^t)$ since the attacker controls these malicious clients. 
    $\mathbf{I}^t\in \{0,1\}^{|\mathcal{S}^t|}$ is the identity vector for the random client subset $\mathcal{S}^t \subseteq \{1, \ldots, n \}$, where the identities of malicious and benign devices are $1$ and $0$ respectively. Notice that the clients' identities are independent of players' actions.
    \item $A = \{ A_{\mathcal{D}}, A_\xi \}$ is the joint action set, where $A_\D$ and $A_{\xi}$ denote the set of defense actions and type-$\xi$ attack actions, respectively; in the FL setting, $a_\D^t=\widehat{w}_g^{t+1}:=h({w}_g^{t+1})$. The attacker's action is characterized by the joint actions of malicious clients $a^t_{A_{\xi}}:=\{\widetilde{g}_i^t\}_{i=1}^{M_1}\cup \{\widetilde{g}_i^t\}_{i=M_1+1}^{M_2}$.  Note that a malicious device not sampled at round $t$ does not send any information to the server; hence its action has no effect on the model update. The subscript $\xi$ is suppressed if it is clear from the context. 
    \item $\mathcal{T}: S\times A \rightarrow \Delta(S)$ 
    is the state transition function, which represents the probability of reaching a state $s'\in S$ from current state $s\in S$, where the defender and the attacker chose actions $a^t_{\D}$ and $a^t_{A_{\xi}}$ respectively.
    \item $r=\{r_{\mathcal{D}},r_{\mathcal{A}_\xi}\}$, where  $r_{\mathcal{D}}: S \times A \rightarrow \mathbb{R}_{\leq 0}$ and $r_{\mathcal{A}_\xi}: S \times A \rightarrow \mathbb{R}$ are the reward functions for the defender and the attacker, respectively. Define the expected reward at round $t$ as $r_{\mathcal{D}}^t:= -\mathbb{E}[F(\widehat{w}^{t+1}_g)]$ and $r_{\mathcal{A}_{\xi}}^t:=\rho\mathbb{E}[F'(\widehat{w}^{t+1}_g)]-(1-\rho)\mathbb{E}[F(\widehat{w}^{t+1}_g)]$, $\rho=M_1/(M_1+M_2)$, if $\mathbf{1}\cdot \mathbf{I}^t > 0$, and $r^t_{\mathcal{A}_{\xi}}:=0$ otherwise. 
\end{itemize}

\section{Meta-Stackelberg Equilibrium}

Let the defender's and the attacker's policies be parameterized by neural networks $\pi_\D(a_\D^t|s^t;\theta)$, $\pi_\A(a_\A^t|s^t;\phi)$  with model weights $\theta\in \Theta$ and $\phi\in \Phi$, respectively. Given the two players' policies $\theta$, $\phi$ and the private attack type $\xi$, the defender's expected utility is defined as $J_\D(\theta, \phi, \xi):=\E_{a_\A^t\sim \pi_\A(\cdot;\phi, \xi), a_\D^t\sim \pi_\D(\cdot;\theta)}[\sum_{t=1}^H \gamma^t r_\D(s^t, a_\D^t, a_\A^t)]$. Similarly, the attacker's expected utility is $J_\A(\theta, \phi, \xi):=\E_{a_\A^t\sim \pi_\A(\cdot;\phi, \xi), a_\D^t\sim \pi_\D(\cdot;\theta)}[\sum_{t=1}^H \gamma^t r_\A(s^t, a_\D^t, a_\A^t)]$. Denote by $\tau_\xi:=  (s^k, a^k_\D, a^k_{\A})_{k=1}^H$ the trajectory of the BSMG under type-$\xi$ attacker, which is subject to the distribution $q(\theta, \phi, \xi):=\prod_{t=1}^H \pi_\D(a_\D^t|s^t;\theta)\pi_\A(a^t_\A|s^t;\phi, \xi)\mathcal{T}(s^{t+1}|s^t, a_\D^t, a_\A^t)$. In the later development of meta-SG, we consider the gradient $\nabla_\theta J_\D(\theta,\phi,\xi)$ and its sample estimate $\hat{\nabla}_\theta J_\D(\tau_\xi)$ based on the trajectory $\tau_\xi$. 
The estimation is due to the policy gradient theorem \cite{sutton_PG} reviewed in Appendix D, and we note that such an estimate takes a batch of $\tau_\xi$ (the batch size is $N_b$) for variance reduction.
For simplicity, we use the one-trajectory estimate denoted by $\hat{\nabla}_\theta J_\D(\tau_\xi)$.

A natural defense strategy to tackle the information asymmetry is to find a Bayesian Stackelberg equilibrium (BSE):  
\begin{equation}
\begin{aligned}
      & \quad \max_{\theta\in \Theta}\E_{\xi\sim Q(\cdot)}[J_\D(\theta, \phi^*_\xi, \xi)]\quad \\
     &  \text{s.t. } \phi_\xi^* \in \argmax J_\A(\theta, \phi, \xi), \forall \xi\in \Xi. \label{eq:bse}
\end{aligned}
\end{equation}
 \eqref{eq:bse} admits a simple characterization for optimal defense, yet its limitation is evident. The attacker's actions (equivalently, the aggregated models) reveal partial information about its hidden type (its attack objective), which the defender does not properly handle, as the strategy is fixed throughout the BSMG. 
 Consequently, the defender does not adapt to the specific attacker in the online execution.

To equip the defender with responsive intelligence in the face of unknown multi-type attacks, we propose a new equilibrium concept, meta-Stackelberg equilibrium in \Cref{def:meta-se}. 
The intuition of this meta-equilibrium is that $\Psi(\theta,\tau_\xi)$ is tailored to each realized $\xi$  when the defender observes the attacker's moves included in $\tau_\xi$. 
\begin{definition}[Meta Stackelberg Equilibrium]
\label{def:meta-se}
    A triple of the defender's meta policy $\theta$, the adaptation mapping $\Psi$, and the attacker's type-dependent policy $\phi$ is a meta Stackelberg equilibrium if it satisfies 
\begin{equation}
\begin{aligned}
    \label{eq:meta-se}
    & \quad \max_{\theta\in\Theta, \Psi} V(\theta):= \mathbb{E}_{\xi\sim Q}\mathbb{E}_{\tau\sim q}[J_\D(\Psi(\theta,\tau),\phi^*_{\xi}, \xi)], \\ & \text{s.t. }  \phi^*_{\xi}\in \arg\max \mathbb{E}_{\tau\sim q}J_\A(\Psi(\theta, \tau), \phi, \xi), \forall \xi\in \Xi, 
\end{aligned}
\end{equation}
where $q=q(\theta, \phi, \xi)$ is the trajectory distribution. 
\end{definition}
In practice, $\Psi(\theta,\tau)$ is simply fixed as a one-step (or multi-step, see \Cref{app:algo}) SGD operation, i.e., $\Psi(\theta,\tau)=\theta+\eta \hat{\nabla}_\theta J_\D(\tau)$ to leave $\theta$ as the only variable to be optimized. In comparison with meta-defense and BSE-defense, the proposed meta-SE defense highlights \textbf{strategic adaptation} in adversarial FL modeled by the BSMG. A detailed discussion on this meta-equilibrium is deferred to \Cref{app:meta-se}.

\section{Meta-Stackelberg Learning}
Based on the aforementioned meta-Stackelberg equilibrium, we introduce the meta-learning-based defense approach~\citep{li2023robust} (referred to as the meta-defense in the sequel) by considering non-adaptive attack methods.
The goal of meta-defense is to find a meta-policy and an adaptation rule such that the adapted policy gives satisfying defense performance. The mathematical characterization is presented in \Cref{app:algo}.

The meta-defense framework includes three stages: \textbf{pre-training}, \textbf{online adaptation}, and 
\textbf{post-training}. The \textbf{pre-training} stage is implemented in a simulated environment (as discussed in 
our technical report), which allows sufficient alternative training with trajectories generated from random potential attacks, which includes both adaptive (e.g., RL-based attacks as discussed in 
our technical report) and non-adaptive (e.g., IPM and LMP) attacks. After obtaining a meta-policy, the defender will interact with the real FL environment in the \textbf{online adaptation} stage to tune its defense policy using feedback (i.e., rewards) received in the face of real attacks. In real-world FL training, the server typically waits for $1\sim 10$ minutes before receiving responses from the clients, allowing the defender to update the defense policy using SGD in the interim. In the \textbf{post-training} stage, the defender finally performs a post-training defense on the global model. 
\subsection{Robust FL via Meta-Stackelberg Learning}
\label{subsec:meta-sl}
Even though meta-learning enables an adaptable defense, it fails to address adaptative attacks in the online adaptation phase, e.g., the RL attacker learns to evade the adapted defense. To create \textbf{strategic} online adaptation, the defender needs to learn the Stackelberg equilibrium of the BSMG using meta-learning, preparing itself for the worst possible case: adaptive attacks in the online phase.

Inspired by a first-order meta-learning algorithm called Reptile~\cite{nichol2018first}, we propose Reptile meta-Stackelberg learning (meta-SL) defense (Algorithm~\ref{algo:meta-sl} solving for the BSMG) in adversarial FL. The algorithms start from 
an initial meta-defense policy model $\theta^0$ and a set of initial attack policies $\{\phi_i\}_{i\leq m}$ from the attack domain. Meta-SL alternatively updates the defender's and attacker's policies, and each iteration comprises three major steps. \textbf{Step 1 attack sampling}: in each iteration, the algorithm first sample a batch of $K$ attacks from $Q(\xi)$. \textbf{Step 2 policy adaptation}: for each attack, meta-SL adapts the defender's policy $\theta^t_k$ to the $k$-th attack by performing $l$-step gradient descent to the defender's value function (see Appendix~\ref{app:algo}) with step size $\eta$. \textbf{Step 3 meta policy update}: meta-SL first derives the attacker's best response policy for each sampled type by applying gradient descent to the attacker's value function until convergence. After this alternative learning, the algorithm then performs $l$-step gradient descent to the defender's value function under the convergent attack policy (of each type) for adaptation. Finally, the average of all adapted policies of the sampled attacks becomes the new meta policy.

\subsection{First-order Approximate Meta Equilibrium}
 
Now we unfold the theoretical analysis for the \textbf{pre-training} stage, which we refer to as meta-Stackelberg learning (meta-SL). 
The main task of meta-SL is solving \eqref{eq:meta-se}, a bi-level optimization problem. We employ a bi-level approach, applying gradient ascent to the upper-level problem (the defender's) where the gradient estimation involves the optimizer of the lower-level problem (the attacker's).
The details are deferred to \Cref{app:algo}.

In general, the meta-SE (\Cref{def:meta-se}) may not be feasible~\citep{nouiehed2019solving}, we hereby propose a weaker characterization that only involves the first-order necessary conditions. 
  To simplify our exposition, we let $\LL_\D (\theta, \phi, \xi) := \E_{\tau \sim q} J_{\D}(\theta+\eta \hat{\nabla}_{\theta} J_\D (\tau), \phi, \xi) $ 
 and $\LL_\A (\theta, \phi, \xi) :=  \E_{\tau \sim q} J_\A ( \theta + \hat{\nabla}_\theta J_\D(\tau), \phi, \xi )$, for a fixed type $\xi \in \Xi$.. In the sequel, we will assume $\LL_\D$ and $\LL_\A$ to be continuously twice differentiable and Lipschitz-smooth with respect to both $\theta$ and $\phi$ as in~\citep{li2022sampling}, and the Lipschitz assumptions are deferred to \Cref{app:theory}. 
 
 \begin{definition}[$\varepsilon$-meta First-Order Stackelberg Equilibrium] \label{def:meta-fose}
For a small $\varepsilon \in (0,1)$, a set of parameters $(\theta^*,\{\phi^*_\xi\}_{\xi \in \Xi}) \in \Theta \times \Phi^{|\Xi|}$ is a $\varepsilon$-\textit{meta First-Order Stackelbeg Equilibrium} ($\varepsilon$-meta-FOSE) of the meta-SG if it satisfies the following conditions for  $\xi \in \Xi$, 
\begin{equation} \label{eq:meta-fose}
\begin{aligned}
    \max_{\theta \in \Theta \bigcap B(\theta^*)} \langle \nabla_{\theta} \LL_\D (\theta^*, \phi^*_\xi, \xi), \theta - \theta^*\rangle &\leq \varepsilon, \\ 
       \max_{\phi \in \Phi \bigcap B(\phi^*_\xi )} \langle \nabla_{\phi} \LL_\A (\theta^*, \phi^*_\xi, \xi) , \phi - \phi^*_\xi \rangle &\leq \varepsilon ,
\end{aligned}
\end{equation} 
where $B( \theta^* ): =  \{ \theta \in \Theta : \| \theta - \theta^*\| \leq 1\}$, and  $B( \phi^*_\xi ): =  \{ \theta \in \Theta : \| \phi - \phi^*_\xi \| \leq 1\} $. When $\varepsilon = 0$, the parameter set $(\theta^*,\{\phi^*_\xi\}_{\xi \in \Xi})$ is said to be the meta-FOSE. 
 \end{definition}
The necessary equilibrium condition for  \Cref{def:meta-se} can be reduced to $\| \nabla_{\theta} \LL_\D (\theta^*, \phi_\xi, \xi)\| \leq \varepsilon$ and $\|\nabla_{\phi} \LL_\A (\theta^*, \phi_\xi, \xi)\| \leq \varepsilon$ in the unconstraint settings. 
 Since we utilize stochastic gradient in practice, all inequalities mentioned above shall be considered in expectation.  These conditions, along with the positive-semi-definiteness of the Hessians, construct the optimality conditions for a local solution for the meta-SE, which may not exist even in the zero-sum cases \cite{Jin2019MinmaxOS}. 
The advantage of considering meta-FOSE is that its existence is guaranteed. 
 \begin{theorem} \label{thm:existence}
     Under the condition that $\Theta$ and $\Phi$ are compact and convex, the meta-SG admits at least one meta-FOSE.
 \end{theorem}
For the rest of this section, we assume the attacker is unconstrained, i.e., $\Phi$ is a finite-dimensional Euclidean space.

\subsection{Sufficiency for First-Order Estimation in Strictly Competitive Games}
Finding a meta-FOSE for \eqref{eq:meta-se} is challenging due to the non-convex equilibrium constraint at the lower level. To see this more clearly, consider differentiating the defender's value function: $\nabla_{\theta}V = \E_{\xi \sim Q } [\nabla_{\theta} \LL_\D (\theta, \phi_\xi, \xi) + (\nabla_{\theta} \phi_\xi(\theta))^{\top}\nabla_{\phi}\LL_\D(\theta, \phi_\xi, \xi)]$, where $\nabla_{\theta}\phi_\xi(\cdot)$ is locally characterized by the implicit function theorem, i.e., $\nabla_{\theta}\phi_\xi (\theta) = ( -\nabla^2_{ \phi}\LL_\A (\theta, \phi, \xi))^{-1}  \nabla^2_{ \phi \theta} \LL_\A (\theta, \phi, \xi)$. 
Therefore, the gradient estimation requires iteratively estimating the second-order information for the attacker (lower level) objective, which can be costly and prohibitive in many scenarios \cite{song2019maml}. Hence, we introduce the following assumption to bypass the technicality involved in calculating $\nabla_{\theta}\phi_\xi$.
\begin{assumption}[Strict-Competitiveness]
\label{ass:sc}
    The BSMG is strictly competitive, i.e., there exist constants $c<0$, $d$ such that $\forall \xi \in \Xi$, $s \in S$, $a_\D, a_\A \in A_\D \times A_\xi$, $r_\D(s,a_\D, a_\A)=c r_\A(s,a_\D, a_\A)+d$.
\end{assumption}
The above assumption is a direct extension of the strict-competitiveness (SC) notion in matrix games \cite{adler09compete}. One can treat the SC notion as a generalization of zero-sum games: if one joint action $(a_\D, a_\A)$ leads to payoff increases for one player, it must decrease the other's payoff. In adversarial FL, the untargeted attack naturally makes the game zero-sum (hence, SC), and the backdoor attack also leads to the SC (see \Cref{app:theory}). The purpose of introducing \Cref{ass:sc} is to establish the Danskin-type result \cite{danskin-type} for the Stackelberg game with nonconvex value functions (see \Cref{lem:liplemma}), which spares us from the Hessian inversion. 

Another key regularity assumption we impose on the nonconvex value functions is adapted from the Polyak-Łojasiewicz (PL) condition \cite{karimi2016linear}, which is customary in nonconvex analysis. 
\begin{assumption}[Stackelberg Polyak-Łojasiewicz condition] \label{plass}
There exists a positive constant $\mu$ such that for any $(\theta, \phi) \in \Theta\times \Phi$ and $\xi \in \Xi$, the following inequalities hold: $\frac{1}{2\mu} \|\nabla_{\phi } \mathcal{L}_{\D} (\theta, \phi, \xi)\|^2 \geq  \max_{\phi} \LL_\D (\theta, \phi, \xi) -  \mathcal{L}_{\D}(\theta, \phi, \xi)$, $\
    \frac{1}{2\mu} \|\nabla_{\phi } \LL_{\A} (\theta, \phi, \xi)\|^2 \geq 
   \max_{\phi} \LL_\A (\theta, \phi, \xi) -  \LL_\A(\theta, \phi, \xi) $.
\end{assumption}

Under \Cref{plass}, the first-order estimation is sufficient by \Cref{lem:liplemma}. 
\begin{lemma}\label{lem:liplemma}
    Under Assumptions \ref{plass} and regularity conditions, there exists $\{ \phi_\xi: \phi_\xi \in \arg\max_{\phi}  \LL_{\A} (\theta, \phi, \xi) \}_{\xi \in \Xi}$, such that $\nabla_{\theta}  V(\theta) =  \nabla_{\theta} \mathbb{E}_{\xi \sim Q, \tau \sim q} J_\D (\theta + \eta \hat{\nabla}_{\theta} J_{\D}(\tau), \phi_\xi, \xi)$. Moreover, there exists a constant $L > 0$ such that the defender value function $V(\theta)$ is $L$-Lipschitz-smooth. 
\end{lemma}

\subsection{Non-Asymptotic Iteration Complexity}

{We now present the main iteration complexity results. \Cref{lemma:attackstablize} states that one can stabilize the lower-level simulated RL attacks under mild conditions,  with proper choices of the batch size $(\mathcal{O}( \varepsilon^{-4}))$ and the attacker learning iteration $(\mathcal{O}(\log \varepsilon^{-1}))$.
Moreover, the defender's gradient feedback can be approximated using the last iterate of the inner loop. }

\begin{lemma}\label{lemma:attackstablize}
Under Assumption \ref{plass} and regularity assumptions. For any given $\varepsilon \in (0,1)$, at any iteration $t \in 1, \ldots, N_\D$, if the attacker learning iteration $N_\A$ and the batch size $N_b$ are large enough such that $N_\A  \sim  \mathcal{O}( \log \varepsilon^{-1} )$ and $N_b  \sim \mathcal{O}(\varepsilon^{-4})$, then, for any $\xi \in \Xi$, the attack policy is stabilized, i.e., $$\E \left[ \max_\phi\langle \nabla_{\phi} \LL_\D (\theta^t, \phi^t_{\xi}(N_\A), \xi), \phi - \phi^t_\xi(N_\A)\rangle \right] \leq \varepsilon.$$
 Further, the defender's gradient feedback can be $\varepsilon$-approximated, i.e., 
 $$\E\left[ \| \nabla_{\theta} V(\theta^t) - \nabla_{\theta} \E_{\xi \sim Q}\LL_\D (\theta^t, \phi^t_\xi(N_\A), \xi)\|\right] \leq  \varepsilon, $$ 
 where the expectation $\E[\cdot]$ is taken over all the randomness from the algorithm.
\end{lemma}


Equipped with Lemma \ref{lemma:attackstablize}, we apply the standard analysis for first-order methods in a non-convex setting to the outer loop, leading to the main complexity result in \Cref{thm:main}.

\begin{theorem}\label{thm:main}
     Under assumption \ref{plass} and regularity assumptions, for any given $\varepsilon \in (0,1)$, let the learning rates $\kappa_\A$ and $\kappa_\D$ be properly chosen (see Appendix D); let $N_\A$ and $N_b$ be chosen as required by Lemma \ref{lemma:attackstablize}, then, meta-SL finds a $\varepsilon$-meta-FOSE within $N_\D \sim \mathcal{O}(\varepsilon^{-2})$ iterations. 
\end{theorem}

{\Cref{thm:main} implies that meta-SL requires $\mathcal{O}(\varepsilon^{-2}\log \varepsilon^{-1})$ iterations of first-order evaluation, which matches the $\Omega(\varepsilon^{-2})$ lower bound for the general non-convex smooth setting \cite{carmon2020lower}, (up to a logarithmic factor) indicating the efficiency of meta-SL.}

\section{Experiments}\label{sec:exp}
For the detailed setup of the experiment and corresponding results, please refer to 
our technical report~\cite{li2023first}. In our experiments, we evaluate our meta-SG defense using the MNIST~\cite{lecun1998gradient} and CIFAR-10~\cite{krizhevsky2009learning} datasets. The evaluation is performed under a range of advanced attacks, including non-adaptive and adaptive untargeted model poisoning attacks (specifically, IPM~\cite{xie2020fall}, LMP~\cite{fang2020local}, RL~\cite{li2022learning}), backdoor attacks (BFL~\cite{bagdasaryan2020backdoor}, BRL~\cite{li2023learning}), and a combination thereof. Various robust defenses are taken into account as baselines, including training-stage defenses such as Krum~\cite{blanchard2017machine}, Clipping Median~\cite{yin2018byzantine,sun2019can,li2022learning}, FLTrust~\cite{cao2020fltrust}, and post-training defenses like Neuron Clipping~\cite{wang2022universal}, Pruning~\cite{wu2020mitigating}. 

\begin{figure}
   \centering
   \begin{subfigure}{0.225\textwidth}
   \centering
   \includegraphics[width=1\textwidth]{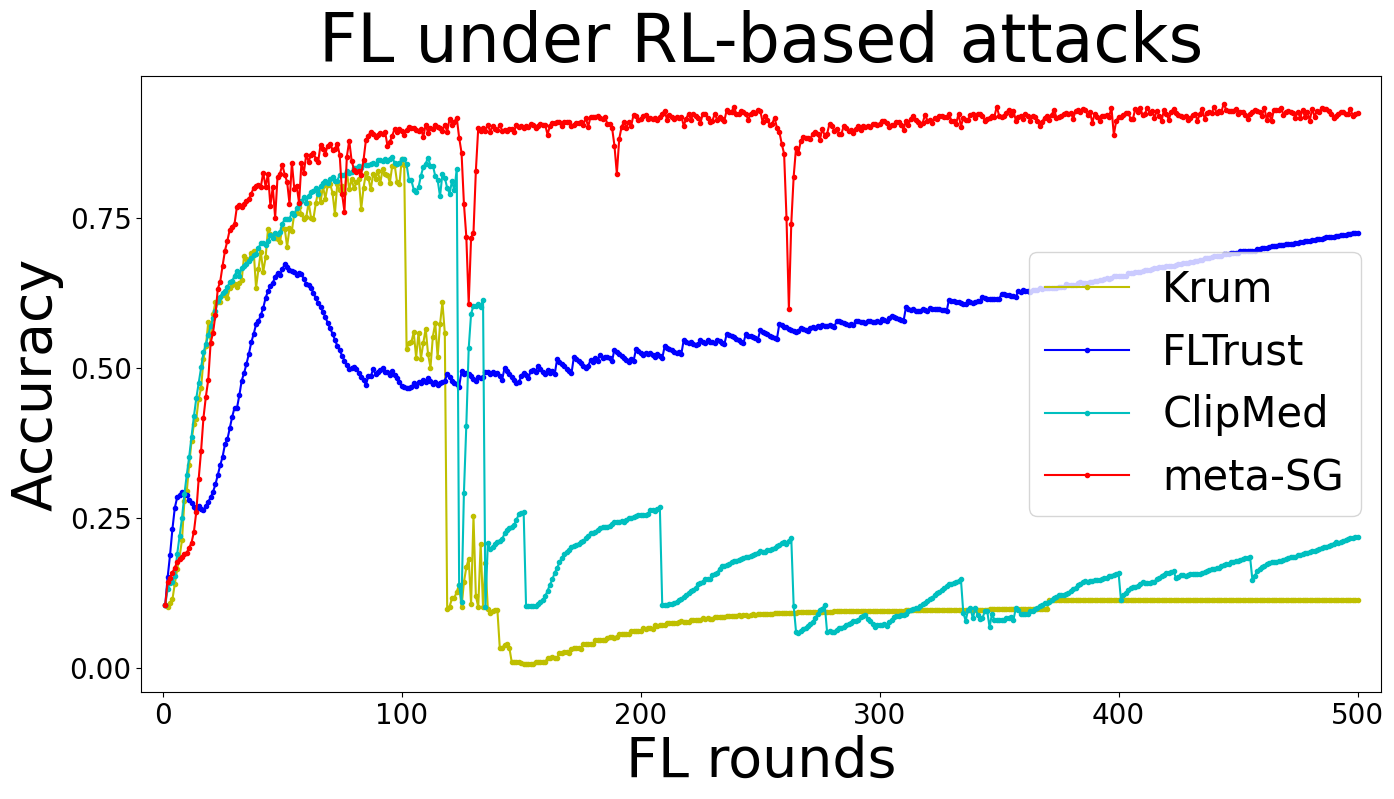}   
   \caption{}
   \end{subfigure}
   \hfill
      \begin{subfigure}{0.225\textwidth}
   \centering
   \includegraphics[width=1\textwidth]{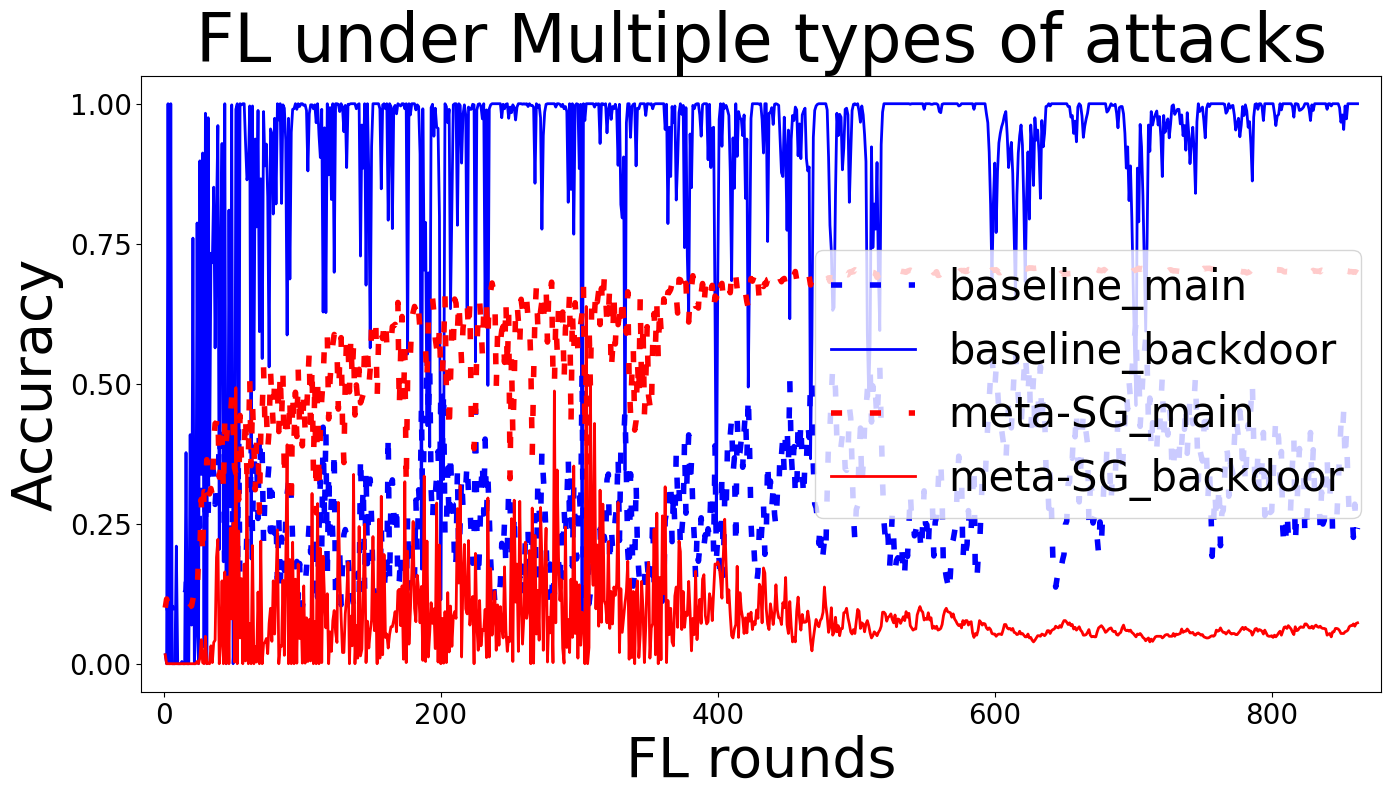}  
   \caption{}
   \end{subfigure}
   \caption{\small{Advantages of the Meta-SG framework against (a) the RL-based model poisoning attack~\cite{li2022learning} on MNIST with $20\%$ malicious devices and (b) a mix of the backdoor attack against FL (BFL)  ~\cite{bagdasaryan2020backdoor} ($5\%$ malicious devices) and the inner product manipulation (IPM) based model poisoning attack~\cite{xie2020fall} ($10\%$ malicious devices) on CIFAR-10. The baseline defense combines the training-stage FLTrust and the post-training Neuron Clipping. }}\label{fig:meta-SG}
   \vspace{-0.7cm}
\end{figure}

As shown in Figure~\ref{fig:meta-SG}(a), meta-SG demonstrates excellent accuracy in federated learning models when facing the RL-based model poisoning attack. Furthermore, as seen in Figure~\ref{fig:meta-SG}(b), meta-SG maintains high model accuracy for unpoisoned data and significantly lowers backdoor accuracy in scenarios involving both backdoor and model poisoning attacks. In contrast, the baseline defense that merely combines a training-stage defense with a post-training defense leads to low model accuracy and fails to shield the FL system from a backdoor attack. {Due to the page limit, more experiment results are deferred to our technical report~\cite{li2023first}.}

\section{Conclusion and Future Work}

In this work, we have proposed a data-driven approach to tackle information asymmetry in adversarial federated learning, which can also be applied to a variety of scenarios in adversarial machine learning, where information asymmetry is prevalent. We have offered a meta-equilibrium solution concept that is computationally tractable and strategically adaptable. In addition, theoretical guarantees on sample complexity and convergence rate have been established under mild assumptions, indicating the computational efficiency of training the meta-equilibrium-based FL defense.

\bibliography{icml2023}
\bibliographystyle{icml2023}

\newpage
\appendix
\onecolumn

\section{Broader Impacts and Limitations}
\textbf{Meta Equilibrium and Information Asymmetry.} Information asymmetry is a prevailing phenomenon arising in a variety of contexts, including adversarial machine learning (e.g. FL discussed in this work), cyber security \citep{manshaei13gamesec}, and large-scale network systems \citep{tao22confluence}. Our proposed meta-equilibrium (\cref{def:meta-se}) offers a data-driven approach tackling asymmetric information structure in dynamic games without Bayesian-posterior beliefs (see \Cref{app:meta-se}). Achieving the strategic adaptation through stochastic gradient descent, the meta-equilibrium is computationally superior to perfect Bayesian equilibrium and better suited for real-world engineering systems involving high-dimensional continuous parameter spaces. It is expected that the meta-equilibrium can also be relevant to other adversarial learning contexts, cyber defense, and decentralized network systems.         

\textbf{First-order Method with Strict Competitiveness.} 
Due to the hardness of the stochastic bilevel optimization problem, we have expanded our search scope with an alternative solution concept that merely involves the first-order necessary conditions for meta-SE. 
Our analytical result relies on the special game structure induced by the strict competitiveness assumption, which essentially ``aligns'' the defender/attacker objectives leveraging the nature of policy gradient, despite them being general-sum. 
Relaxing this assumption allows our framework to deal with a more general class of problems, yet may potentially disrupt the Danskin-type structure of gradient estimation. 
For simplicity of exposition, we neglected the stochastic analysis for the defender policy gradient estimation in the outer loop of the algorithm, the concentration of which depends on the trajectory batch size and attacker-type sample size. We leave the outer loop sample-complexity analysis to future work.

\textbf{Incomplete Universal Defense.} Our aim is to establish a comprehensive framework for universal federated learning defense. This framework ensures that the server remains oblivious to any details pertaining to the environment or potential attackers. Still, it possesses the ability to swiftly adapt and respond to uncertain or unknown attackers during the actual federated learning process. Nevertheless, achieving this universal defense necessitates an extensive attack set through pre-training, which often results in a protracted convergence time toward a meta-policy. Moreover, the effectiveness and efficiency of generalizing from a wide range of diverse attack distributions 
pose additional challenges. Considering these, we confine our experiments in this paper to specifically address a subset of uncertainties and unknowns. This includes 
parameters that determine
attack types, the number of malicious devices, 
the heterogeneity of local data distributions, backdoor triggers, backdoor targets, and other relevant aspects. However, we acknowledge that our focus is not all-encompassing, and there may be other factors that remain unexplored in our research.

\section{Further Justification on Meta Equilibrium}\label{app:meta-se}
\setcounter{equation}{0}
\renewcommand{\theequation}{\Alph{section}\arabic{equation}}
This section offers further justification for the meta-equilibrium, and we argue that meta-equilibrium provides a data-driven approach to address incomplete information in dynamic games. Note that information asymmetry is prevalent in the adversarial machine learning context, where the attacker enjoys an information advantage (e.g., the attacker's type). The proposed meta-equilibrium notion can shed light on these related problems beyond the adversarial FL context. 

We begin with the insufficiency of Bayesian Stackelberg equilibrium \eqref{eq:bse} in handling information asymmetry, a customary solution concept in security studies \cite{tao_info}. One can see from \eqref{eq:bse} that such an equilibrium is of ex-ante type: the defender's strategy is determined before the game starts. It targets an ``representative'' attacker (an average of all types). As the game unfolds, new information regarding the attacker's private type is revealed (e.g., through the global model updates). However, this ex-ante strategy does not enable the defender to handle this emerging information as the game proceeds. Using game theory language, the defender fails to adapt its strategy in the interim stage. 

To create interim adaptability in this dynamic game of incomplete information, one can consider introducing the belief system to capture the defender's learning process on the hidden type. Let $I^t$ be the defender's observations up to time $t$, i.e., $I^t:=(s^k, a_{\mathcal{D}}^k)_{k=1}^t s^{t+1}$. Denote by $\B$ the belief generation operator $b^{t+1}(\xi)=\B[I^t]$. With the Bayesian equilibrium framework, the belief generation can be defined recursively as below
\begin{align}
    b^{t+1}(\xi)=\B[s^t,a^t_\D, b^t]:=\frac{b^{t}(\xi)\pi_\A(a^t_\A|s^t;\xi)\mathcal{T}(s^{t+1}|s^t,a^t_\A,a^t_\D)}{\sum_{\xi'}b^{t}(\xi')\pi_\A(a^t_\A|s^t;\xi')\mathcal{T}(s^{t+1}|s^t,a^t_\A,a^t_\D)}.
    \label{eq:belief}
\end{align}
Since $b^t$ is the defender's belief on the hidden type at time $t$, its belief-dependent Markovian strategy is defined as $\pi_\D(s^t, b^t)$. Therefore, the interim equilibrium, also called Perfect Bayesian Equilibrium (PBE) \cite{fudenberg} is given by a tuple $(\pi^*_\D, \pi^*_\A, \{b^t\}_{t=1}^H)$ satisfying 
\begin{equation}
  \begin{aligned}
    & \pi_\D^*=\arg\max \E_{\xi\sim Q}\mathbb{E}_{\pi_\D, \pi^*_\A}[\sum_{t=1}^H r_\D(s^t, a_\D^t,a_\A^t)b^t(\xi)]\\
    & \pi_\A^*=\arg\max \mathbb{E}_{\pi_\D, \pi_\A}[\sum_{t=1}^H r_\A(s^t, a_\D^t, a_\A^t)], \forall \xi,\\
    & \{b^k\}_{k=1}^H \text{ satisfies } \eqref{eq:belief} \text{ for realized actions and states}.
\end{aligned}\tag{PBE}\label{eq:mpbe}  
\end{equation}
In contrast with \eqref{eq:bse}, this perfect Bayesian equilibrium notion \eqref{eq:mpbe} enables the defender to make good use of the information revealed by the attacker, and subsequently adjust its actions according to the revealed information through the belief generation. From a game-theoretic viewpoint, both \eqref{eq:mpbe} and \eqref{eq:meta-se} create strategic online adaptation: the defender can infer and adapt to the attacker's private type through the revealed information since different types aim at different objectives, hence, leading to different actions. Compared with PBE, the proposed meta-equilibrium notion is better suited for large-scale complex systems where players' decision variables can be high-dimensional and continuous, as argued in the ensuing paragraph. 

To achieve the strategic adaptation, PBE relies on the Bayesian-posterior belief updates, which soon become intractable as the denominator in \eqref{eq:belief} involves integration over high-dimensional space and discretization inevitably leads to the curse of dimensionality. Despite the limited practicality, PBE is inherently difficult to solve even in finite-dimensional cases. It is shown in \cite{bhaskar2016hardness} that the equilibrium computation in games with incomplete information is NP-hard, and how to solve for PBE in dynamic games remains an open problem. Even though there have been encouraging attempts at solving PBE in two-stage games \cite{tao23pot}, it is still challenging to address PBE computation in generic Markov games.


\section{Algorithms}\label{app:algo}
\newcommand{\tp}{\mathsf{T}}
\setcounter{equation}{0}
\renewcommand{\theequation}{\Alph{section}\arabic{equation}}
This section elaborates on meta-learning defense and meta-Stackelberg learning.  To begin with, we first review the policy gradient method \cite{sutton_PG} in RL and its Monte-Carlo estimation. To simplify our exposition, we fix the attacker's policy $\phi$, and then BSMG reduces to a single-agent MDP, where the optimal policy to be learned is the defender's $\theta$. 
\paragraph{Policy Gradient}
The idea of the policy gradient method is to apply gradient ascent to the value function $J_\D$. Following \cite{sutton_PG}, we obtain $\nabla_\theta J_\D:=\E_{\tau\sim q(\theta)}[g(\tau;\theta)]$, where $g(\tau;\theta)=\sum_{t=1}^H \nabla_\theta \log \pi(a_\D^t|s^t;\theta)R(\tau)$ and $R(\tau)=\sum_{t=1}^H \gamma^t r(s^t, a_{\D}^t)$. Note that for simplicity, we suppress the parameter $\phi, \xi$ in the trajectory distribution $q$, and instead view it as a function of $\theta$. In numerical implementations, the policy gradient $\nabla_\theta J_\D$ is replaced by its Monte-Carlo (MC) estimation using sample trajectory. Suppose a batch of trajectories $\{\tau_i\}_{i=1}^{N_b}$, and $N_b$ denotes the batch size, then the MC estimation is 
\begin{equation}
    \hat{\nabla}_\theta J_\D(\theta,\tau):=1/{N_b} \sum_{\tau_i} g(\tau_i;\theta)
    \label{eq:mc-pg}
\end{equation}
 The same deduction also holds for the attacker's problem when fixing the defense $\theta$. 

\paragraph{Meta-Learning FL Defense}
Meta-learning-based defense (meta defense) mainly targets non-adaptive attack methods, where $\pi_\A(\cdot;\phi,\xi)$ is a pre-fixed attack strategy following some rulebook, such as IPM \cite{xie2020fall} and LMP \cite{fang2020local}. In this case, the BSMG reduces to single-agent MDP for the defender, where the transition kernel is determined by the attack method.   Mathematically, the meta-defense problem is given by 
\begin{equation}
    \max_{\theta, \Psi} \E_{\xi\sim Q(\cdot)}[J_\D(\Psi(\theta,\tau), \phi, \xi)]
    \label{eq:meta-defense}
\end{equation}
Since the attack type is hidden from the defender, the adaptation mapping $\Psi$ is usually defined in a data-driven manner. For example, $\Psi(\theta, \tau)$ can be defined as a one-step stochastic gradient update with learning rate $\eta$: $\Psi(\theta, \tau)=\theta+\eta \hat{\nabla} J_\D(\tau_\xi)$  \cite{finn2017model} or a recurrent neural network in \cite{duan2016rl}. This work mainly focuses on gradient adaptation for the purpose of deriving theoretical guarantees in \Cref{app:theory}.  

With the one-step gradient adaptation, the meta-defense problem in \eqref{eq:meta-defense} can be simplified as 
\begin{equation}
    \max_{\theta} \E_{\xi \sim Q(\cdot)}\E_{\tau\sim q(\theta)}[J_\D(\theta+\eta \hat{\nabla}_\theta J_\D(\tau), \phi, \xi)]
    \label{eq:meta-def-sgd}
\end{equation}
Recall that the attacker's strategy is pre-determined, $\phi,\xi$ can be viewed as fixed parameters, and hence, the distribution $q$ is a function of $\theta$. To apply the policy gradient method to  \eqref{eq:meta-def-sgd}, one needs an unbiased estimation of the gradient of the objective function in \eqref{eq:meta-def-sgd}. Consider the gradient computation using the chain rule:
\begin{equation}
   \begin{aligned}
    &\nabla_\theta \E_{\tau\sim q(\theta)}[J_\D(\theta+\eta \hat{\nabla}_\theta J_\D(\tau), \phi, \xi )]\\
    &=\E_{\tau\sim q(\theta)}\{\underbrace{\nabla_\theta J_\D(\theta+\eta\hat{\nabla}_\theta J_\D(\tau), \phi, \xi)(I+\eta \hat{\nabla}^2_\theta J_D(\tau))}_{\text{\ding{172}}}\\
    &\quad  +\underbrace{J_\D(\theta+\eta \hat{\nabla}_\theta J_\D(\tau))\nabla_\theta \sum_{t=1}^H \pi(a^t|s^t;\theta)}_{\text{\ding{173}}}\}.
\end{aligned} 
\label{eq:chain}
\end{equation}
The first term results from differentiating the integrand $J_\D(\theta+\eta \hat{\nabla}_\theta J_\D(\tau), \phi, \xi )$ (the expectation is taken as integration), while the second term is due to the differentiation of $q(\theta)$. One can see from the first term that the above gradient involves a Hessian $\hat{\nabla}^2 J_\D$, and its sample estimate is given by the following. For more details on this Hessian estimation, we refer the reader to \cite{fallah2021convergence}.
\begin{align}
    \hat{\nabla}^2 J_\D(\tau)=\frac{1}{N_b}\sum_{i=1}^{N_b} [ g(\tau_i;\theta)\nabla_\theta \log q(\tau_i;\theta)^\tp+\nabla_\theta g(\tau_i;\theta)]
    \label{eq:hessian-est}
\end{align}
Finally, to complete the sample estimate of $\nabla_\theta \E_{\tau\sim q(\theta)}[J_\D(\theta+\eta \hat{\nabla}_\theta J_\D(\tau), \phi, \xi )]$, one still needs to estimate $\nabla_\theta J_\D(\theta+\eta\hat{\nabla}_\theta J_\D(\tau), \phi, \xi)$ in the first term. To this end, we need to first collect a batch of sample trajectories ${\tau'}$ using the adapted policy $\theta'=\theta+\eta \hat{\nabla}_\theta J_D(\tau)$. Then, the policy gradient estimate of $\hat{\nabla}_\theta J_\D(\theta')$ proceeds as in \eqref{eq:mc-pg}. To sum up, constructing an unbiased estimate of \eqref{eq:chain} takes two rounds of sampling. The first round is under the meta policy $\theta$, which is used to estimate the Hessian \eqref{eq:hessian-est} and to adapt the policy to $\theta'$. The second round aims to estimate the policy gradient $\nabla_\theta J_\D(\theta+\eta\hat{\nabla}_\theta J_\D(\tau), \phi, \xi)$ in the first term in \eqref{eq:chain}.

In the experiment, we employ a first-order meta-learning algorithm called Reptile~\cite{nichol2018first} to avoid the Hessian computation. The gist is to simply ignore the chain rule and update the policy using the gradient $\nabla_\theta J_\D(\theta',\phi,\xi)|_{\theta'=\theta+\eta\hat{\nabla}_\theta J_\D(\tau)}$. Naturally, without the Hessian term, the gradient in this update is biased, yet it still points to the ascent direction as argued in \cite{nichol2018first}, leading to effective meta policy. The advantage of Reptile is more evident in multi-step gradient adaptation. Consider a $l$-step gradient adaptation, the chain rule computation inevitably involves multiple Hessian terms (each gradient step brings a Hessian term) as shown in \cite{fallah2021convergence}. In contrast, Reptile only requires first-order information, and the meta-learning algorithm ($l$-step adaptation) is given by \Cref{algo:meta-rl}. 

\begin{algorithm}[ht]
\begin{algorithmic}[1]
\STATE \textbf{Input: } the type distribution $Q(\xi)$, step size parameters $\kappa,\eta$
\STATE \textbf{Output: }$\theta^T$
\STATE randomly initialize $\theta^0$
\FOR{iteration $t=1$ to $T$}
\STATE Sample a batch $\hat{\Xi}$ of $K$ attack types from $Q(\xi)$;
\FOR{each $\xi\in \hat{\Xi}$}
\STATE $\theta_\xi^{t}(0) \leftarrow \theta^{t}$
\FOR{$k=0$ to $l-1$}
\STATE Sample a batch trajectories ${\tau}$ of the horizon length $H$ under $\theta^{t}_\xi(k)$;
\STATE Evaluate $\hat{\nabla}_\theta J_\D(\theta^t_\xi(k),\tau)$ using MC in \eqref{eq:mc-pg};
\STATE $\theta_\xi^{t}(k+1) \leftarrow \theta_\xi^{t}(k) + \kappa \hat{\nabla}_\theta J_\D(\theta^t,\tau)$
\ENDFOR
\ENDFOR
\STATE Update $\theta^{t+1} \leftarrow \theta^{t}+1/K \sum_{\xi\in \hat{\Xi}}(\theta^t_\xi(l)-\theta^t)$;
\ENDFOR
\end{algorithmic}
 \caption{Reptile Meta-Reinforcement Learning with $l$-step adaptation}
 \label{algo:meta-rl}
\end{algorithm}

\paragraph{Meta-Stackelberg Learning}
Recall that in meta-SE, the attacker's policy $\phi_\xi^*$ is not pre-fixed, instead, it is the best response to the defender's adapted policy.  To obtain this best response, one needs alternative training: fixing the defense policy, and applying gradient ascent to the attacker's problem until convergence. It should be noted that the proposed meta-SL utilizes the unbiased gradient estimation in \eqref{eq:hessian-est}, which paves the way for theoretical analysis in \Cref{app:theory}. Yet, we turn to the Reptile to speed up pre-straining in the experiments. We present both algorithms in \Cref{algo:meta-sl}, and only consider one-step adaptation for simplicity. The multi-step version is a straightforward extension of \Cref{algo:meta-sl}.   
\begin{algorithm}
 \begin{algorithmic}[1]
     \STATE \textbf{Input: } the type distribution $Q(\xi)$, initial defense meta policy $\theta^0$, pre-trained attack policies $\{\phi_\xi^{0}\}_{\xi\in \Xi}$, step size parameters $\kappa_\D$, $\kappa_\A$, $\eta$, and iterations numbers $N_\A, N_\D$;
     \STATE \textbf{Output: }$\theta^{N_\D}$
     \FOR{iteration $t=0$ to $N_\D-1$}
     \STATE Sample a batch $\hat{\Xi}$ of $K$ attack types from $Q(\xi)$;
     \FOR{each $\xi\in \hat{\Xi}$}
     \STATE Sample a batch of trajectories using $\phi^t$ and $\phi_\xi^t$;
     \STATE Evaluate $\hat{\nabla}_\theta J_D(\theta^t, \phi^t_\xi, \xi)$ using \eqref{eq:mc-pg};
     \STATE Perform one-step adaptation $\theta_\xi^t\gets \theta^t+\eta \hat{\nabla}_\theta J_D(\theta^t_\xi(k), \phi^t_\xi, \xi) $;
     \STATE $\phi^t_\xi(0)\gets \phi^t_\xi$;
     \FOR{$k=0,\ldots, N_\A-1$}
     \STATE Sample a batch of trajectories using $\theta_\xi^t$ and $\phi^t_\xi(k)$;
     \STATE $\phi^{t}_\xi(k+1)\gets \phi^{t}_\xi(k)+\kappa_\A \hat{\nabla}_{\phi} J_\A(\theta_\xi^t, \phi_\xi^t(k), \xi)$;
     \ENDFOR
     \IF{Reptile}
     \STATE Sample a batch of trajectories using $\theta_\xi^t$ and $\phi_\xi^t(N_\A)$;
     \STATE Evaluate $\hat{\nabla}J_D(\xi):= \hat{\nabla}_\theta J_\D(\theta, \phi^t_\xi(N_\A), \xi)|_{\theta=\theta_\xi^t}$ using \eqref{eq:mc-pg};
     \ELSE
     \STATE Sample a batch of trajectories using $\theta^t$ and $\phi^t_\xi(N_\A)$;
     \STATE Evaluate the Hessian using \eqref{eq:hessian-est};
     \STATE Sample a batch of trajectories using $\theta^t_\xi$ and $\phi^t_\xi(N_\A)$; 
     \STATE Evaluate $\hat{\nabla}J_D(\xi):= \hat{\nabla}_\theta J_\D(\theta_\xi^t, \phi^t_\xi(N_\A), \xi)$  using \eqref{eq:chain};
     \ENDIF
     \STATE $\bar{\theta}^t_\xi\gets \theta^t+\kappa_\D \hat{\nabla}J_D(\xi)$;
     \ENDFOR
     \STATE $\theta^{t+1}\gets   \theta^t+1/{K} \sum_{\xi\sim \hat{\Xi}} (\bar{\theta}^t_\xi-\theta_t) $, $\phi^{t+1}_\xi\gets  \phi^t_\xi(N_\A)$;
     \ENDFOR

 \end{algorithmic}
 \caption{(Reptile) Meta-Stackelberg Learning with one-step adaptation}
 \label{algo:meta-sl}
\end{algorithm}

\section{Theoretical Results}
\label{app:theory}

\subsection{Existence of Meta-SG}

\begin{theorem}[\Cref{thm:existence}]
 Under the conditions that $\Theta$ and $\Phi$ are compact and convex, the meta-SG admits at least one meta-FOSE.
\end{theorem}

\begin{proof}
 Clearly, $\Theta \times \Phi^{|\Xi|}$ is compact and convex, let $\phi \in \Phi^{|\Xi|}, \phi_\xi \in \Phi$ be the (type-aggregated) attacker's strategy, since the  consider twice continuously differentiable utility functions $\ell_\D(\theta, \phi) := \E_{\xi \sim Q} \LL_\D (\theta, \phi_\xi, \xi)$ and $\ell_\xi (\theta, \phi) := \LL_\A(\theta, \phi_\xi, \xi)$ for all $\xi \in \Xi$. 
 Then, there exists a constant $\gamma_c > 0$, such that the auxiliary utility functions:
 \begin{equation}\label{auxiliaryutility}
\begin{aligned}
      \tilde{\ell}_\D (\theta; (\theta^{\prime}, \phi^{\prime}))  & \equiv  \ell_\D (\theta, \phi)  - \frac{\gamma_c}{2} \| \theta -  \theta^{\prime}\|^2
      \\  \tilde{\ell}_\xi (\phi_\xi ; ( \theta^{\prime}, \phi^{\prime})& \equiv \ell_\xi(\theta^{\prime}, (\phi_{\xi }, \phi^{\prime}_{-\xi}) )   - \frac{\gamma_c}{2} \| \phi_\xi -  \phi^{\prime}_\xi\|^2 \quad \forall \xi \in \Xi
\end{aligned}
\end{equation}
are $\gamma_c$-strongly concave in spaces $\theta \in \Theta$, $\phi_\xi \in \Phi$ for all $\xi \in \Xi$, respectively for any fixed $(\theta^{\prime}, \phi^{\prime}) \in \Theta \times \Phi^{|\Xi|}$.

Define the self-map $h: \Theta \times \Phi^{|\Xi|} \to \Theta \times \Phi^{|\Xi|}$ with $h(\theta^{\prime},\phi^{\prime}) \equiv ( \bar{\theta}(\theta^{\prime},\phi^{\prime}), \bar{\phi}(\theta^{\prime},\phi^{\prime}))$, where  
\begin{align*}
      \bar{\theta}(\theta^{\prime},\phi^{\prime})  = \argmax_{\theta \in \Theta}  \tilde{\ell}_\D (\theta, \phi^{\prime} ), 
      \quad \quad
     \bar{\phi}_\xi( \theta^{\prime},\phi^{\prime})   = \argmax_{\phi_\xi \in \Phi}  \tilde{\ell}_\xi (\theta^{\prime}, (\phi_{\xi}, \phi^{\prime}_{-\xi})) . 
\end{align*}
Due to compactness, $h$ is well-defined. 
By strong concavity of $ \tilde{\ell}_\D (\cdot; (\theta^{\prime},\phi^{\prime}))$ and $\tilde{\ell}_\xi (\cdot; (\theta^{\prime},\phi^{\prime}))$, it follows that $\bar{\theta}, \bar{\phi }$ are continuous self-mapping from $\Theta \times \Phi^{|\Xi|}$ to itself. By Brouwer's fixed point theorem, there exists at least one $(\theta^*, \phi^*) \in \Theta \times \Phi^{|\Xi|}$ such that $h(\theta^*, \phi^*) = (\theta^*, \phi^*)$. Then, one can verify that $(\theta^*, \phi^*)$ is a meta-FOSE of the meta-SG with utility function $\ell_\D$ and $\ell_\xi$, $\xi \in \Xi$, in view of the following inequality
\begin{align*}
   \langle \nabla_{\theta} \tilde{\ell}_\D (\theta^*; (\theta^*, \phi^*)) ,  \theta  - \theta^* \rangle & = \langle \nabla_{\theta} \ell_\D ( \theta^*, \phi^* )  ,  \theta  - \theta^* \rangle  \\
 \langle \nabla_{\phi_\xi} \tilde{\ell}_\xi (\theta^*; (\theta^*, \phi^*)) ,  \phi_\xi  - \phi^*_\xi \rangle & = \langle \nabla_{\phi_\xi } \ell_\xi ( \theta^*, \phi^* )  ,  \phi_\xi  - \phi^*_\xi \rangle , 
\end{align*}
therefore, the equilibrium conditions for meta-SG with utility functions $\tilde{\ell}_\D$ and $\{\tilde{\ell}_\xi\}_{\xi \in \Xi}$ are the same as with utility functions $\ell_\D$ and $\{\ell_\xi\}_{\xi \in \Xi}$, hence the claim follows.
\end{proof}

\subsection{Proofs: Non-Asymptotic Analysis}

In the sequel, we make the following smoothness assumptions for every attack type $\xi \in \Xi$. In addition, we assume, for analytical simplicity, that all types of attackers are unconstrained, i.e., $\Phi$ is the Euclidean space with proper finite dimension.



     


 \begin{assumption}[($\xi$-wise) Lipschitz smoothness] \label{asslip}
 The functions $\mathcal{L}_{\D}$ and $\LL_\A$ are continuously diffrentiable in both $\theta$ and $\phi$. Furthermore, there exists constants $L_{11}$, $L_{12}, L_{21}$, and $L_{22}$ such that
 for all $\theta, \theta_1, \theta_2 \in \Theta$ and $\phi, \phi_1, \phi_2 \in \Phi$, we have, for any $\xi \in \Xi$,
\begin{align}
    \left\|\nabla_{\theta} \mathcal{L}_\D  \left(\theta_{1}, \phi, \xi\right)-\nabla_{\theta} \mathcal{L}_\D\left(\theta_{2}, \phi, \xi\right)\right\| &  \leq L_{11}\left\|\theta_{1}-\theta_{2}\right\|  \label{lip1}
    \\ 
    \left\|\nabla_{\phi} \mathcal{L}_\D \left(\theta, \phi_{1}, \xi\right)-\nabla_{\phi} \mathcal{L}_\D \left(\theta, \phi_{2}, \xi\right)\right\| & \leq L_{22}\left\|\phi_{1}-\phi_{2}\right\|  \label{lip2}
     \\ 
     \left\|\nabla_{\theta} \mathcal{L}_\D \left(\theta, \phi_{1}, \xi\right)-\nabla_{\theta} \mathcal{L}_\D \left(\theta, \phi_{2}, \xi\right)\right\| & \leq L_{12}\left\|\phi_{1}-\phi_{2}\right\|  \label{lip3}
    \\ 
    \left\|\nabla_{\phi} \mathcal{L}_\D \left(\theta_{1}, \phi, \xi\right)-\nabla_{\phi} \mathcal{L}_\D \left(\theta_{2}, \phi, \xi\right)\right\|  & \leq L_{12}\left\|\theta_{1}-\theta_{2}\right\|  \label{lip4}  
    \\ 
    \| \nabla_{\phi} \LL_\A (\theta, \phi_1, \xi ) - \nabla_{\phi} \LL_\A (\theta, \phi_2, \xi)\| & \leq  L_{21} \| \phi_1 - \phi_2\| . \label{lip5}
\end{align}
\end{assumption}

\begin{lemma}[Implicit Function Theorem (IFT) for Meta-SG] \label{lemma:ift}
 Suppose for $(\bar{\theta}, \bar{\phi}) \in \Theta \times \Phi^{|\Xi|}$, $\xi \in \Xi$ we have $\nabla_{\phi} \LL_\A (\bar{\theta}, \bar{\phi}, \xi)= 0 $ the Hessian $\nabla^2_{\phi} \LL_\A (\bar{\theta}, \bar{\phi}, \xi)$ is non-singular. Then, there exists a neighborhood $B_{\varepsilon}(\bar{\theta}), \varepsilon > 0$ centered around $\bar{\theta}$ and a $C^1$-function $\phi(\cdot): B_{\varepsilon}(\bar{\theta}) \to \Phi^{|\Xi|}$ such that near $(\bar{\theta}, \bar{\phi})$ the solution set $\{ (\theta, \phi) \in \Theta \times \Phi^{|\Xi|}: 
 \nabla_{\phi} \LL_\A (\theta, \phi, \xi)= 0 \}$ is a $C^1$-manifold locally near $(\bar{\theta}, \bar{\phi})$. The gradient $\nabla_{\theta} \phi(\theta)$ is given by $- ( \nabla^2_{ \phi}\LL_\A (\theta, \phi, \xi))^{-1}  \nabla^2_{ \phi \theta} \LL_\A (\theta, \phi, \xi)$.

\end{lemma}

\begin{lemma}\label{liplemma}
    Under assumptions \ref{asslip}, \ref{plass},  there exists $\{ \phi_\xi: \phi_\xi \in \arg\max_{\phi}  \LL_{\A} (\theta, \phi, \xi) \}_{\xi \in \Xi}$, such that 
    $$\nabla_{\theta}  V(\theta) =  \nabla_{\theta} \mathbb{E}_{\xi \sim Q, \tau \sim q} J_\D (\theta + \eta \hat{\nabla}_{\theta} J_{\D}(\tau), \phi_\xi, \xi).$$
    Moreover, the function $V(\theta)$ is $L$-Lipschitz-smooth, where $L = L_{11} + \frac{L_{12}L_{21}}{\mu}$
    \begin{equation*}
        \|  \nabla_{\theta} V(\theta_1) -  \nabla_{\theta} V(\theta_2) \| \leq L \|\theta_1 - \theta_2 \|.
    \end{equation*}
\end{lemma}

\begin{proof}[Proof of Lemma \ref{liplemma}]
    First, we show that for any $\theta_1, \theta_2 \in \Theta, \xi \in \Xi$, and $\phi_1 \in \argmax_{\phi} \LL_\A (\theta_1, \phi, \xi) $, there exists $\phi_2 \in \argmax_{\phi} \LL_\A (\theta_2 , \phi, \xi) $ such that $\|\phi_1 - \phi_2\| \leq \frac{L_{12}}{\mu} \|\theta_1 - \theta_2\|$.
    Indeed, based on smoothness assumption \eqref{lip5} and \eqref{lip4},
    \begin{equation*}
    \begin{aligned}
        \|\nabla_{\phi} \LL_\A (\theta_1, \phi_1, \xi) - \nabla_{\phi} \LL_\A(\theta_2, \phi_1, \xi )\| \leq  L_{21} \| \theta_1 - \theta_2\| ,  
    \\  \|\nabla_{\phi} \LL_\D (\theta_1, \phi_1, \xi) - \nabla_{\phi} \LL_\D(\theta_2, \phi_1, \xi)\| \leq  L_{12} \| \theta_1 - \theta_2\| .
        \end{aligned}
    \end{equation*}
    Since $\phi_2 \in \argmax_{\phi} \LL_\A (\theta_2 , \phi, \xi)$,  $\nabla_{\phi} \LL_\A (\theta_2, \phi_2, \xi ) = 0$. 
    Apply PL condition to $\nabla_{\phi} \LL_\A (\theta, \phi_2, \xi)$, 
    \begin{equation*}
     \begin{aligned}
         \max_{\phi} \LL_\A (\theta_1, \phi, \xi) -  \LL_\A (\theta_1, \phi_2, \xi) & \leq \frac{1}{2 \mu} \| \nabla_{\phi} \LL_\A (\theta_1, \phi_2, \xi)\|^2 
         \\ & = \frac{1}{2 \mu} \| \nabla_{\phi} \LL_\A (\theta_1, \phi_2, \xi) - \nabla_{\phi} \LL_\A (\theta_2, \phi_2, \xi ) \|^2 
         \\ & \leq  \frac{L^2_{21}}{2\mu} \|\theta_1 - \theta_2\|^2 \quad \quad \text{ by \eqref{lip5}.}
     \end{aligned}
    \end{equation*}
    Since PL condition implies quadratic growth, we also have 
    \begin{equation*}
         \LL_\A (\theta_1, \phi_1, \xi) - \LL_\A (\theta_1, \phi_2, \xi) \geq \frac{\mu}{2} \| \phi_1 - \phi_2\|^2.
    \end{equation*}
    Combining the two inequalities above we obtain the Lipschitz stability for $\phi^*_\xi(\cdot)$, i.e., $$ \| \phi_1 - \phi_2 \|\leq \frac{L_{21}}{\mu}\| \theta_1 - \theta_2\|. $$
    Second, show that $\nabla_{\theta} V(\theta)$ can be directly evaluated at $\{\phi^*_\xi \}_{\xi \in \Xi}$. Inspired by Danskin's theorem, we first made the following argument, consider the definition of directional derivative. Let $\ell(\theta,\phi):=\nabla_\theta \mathbb{E}_{\xi,\tau} J_\D(\theta+\eta \hat{\nabla} J_\D(\tau), \xi)$. For a constant $\tau$ and an arbitrary direction $d$, 
\begin{align*}
    & \quad \ell(\theta+\tau d, \phi^*(\theta+\tau d))- \ell(\theta, \phi^*(\theta)))\\
    &=  \ell(\theta+\tau d, \phi^*(\theta+\tau d))- \ell(\theta+\tau d, \phi^*(\theta))
    +\ell(\theta+\tau d, \phi^*(\theta))- \ell(\theta, \phi^*(\theta))\\
    &= \nabla_\phi \ell(\theta+\tau d, \phi^*(\theta))^\top\underbrace{[\phi^*(\theta+\tau d)-\phi^*(\theta))]}_{\Delta \phi} + o(\Delta \phi^2)\\
    &+ \tau \nabla_\theta \ell(\theta, \phi^*(\theta))^T d + o(d^2).
\end{align*}
Hence, a sufficient condition for the first equation is $\nabla_\phi \ell(\theta+\tau d, \phi^*(\theta))=0$, meaning that $\ell_D(\theta, \phi)$ and $\LL_\A(\theta, \phi, \xi)$ share the first-order stationarity at every $\phi$ when fixing $\theta$. Indeed, by \Cref{lemma:ift}, we have, the gradient is locally determined by 
 \begin{equation*}
     \begin{aligned}
        \nabla_{\theta}V &  = \E_{\xi \sim Q } [\nabla_{\theta} \LL_\D (\theta, \phi_\xi, \xi) + (\nabla_{\theta} \phi_\xi(\theta))^{\top}\nabla_{\phi}\LL_\D(\theta, \phi_\xi, \xi)] \\
        & = \E_{\xi \sim Q } \left[\nabla_{\theta} \LL_\D (\theta, \phi_\xi, \xi) - [( \nabla^2_{ \phi}\LL_\A (\theta, \phi, \xi))^{-1}  \nabla^2_{ \phi \theta} \LL_\A (\theta, \phi, \xi)]^{\top}\nabla_{\phi}\LL_\D(\theta, \phi_\xi, \xi) \right] . 
     \end{aligned}
 \end{equation*}

 Given a trajectory $\tau:=(s^1,a_\D^t, a_\A^t, \ldots, a_\D^{H}, a_\A^H, s^{H+1})$, let $R_\D(\tau, \xi):=\sum_{t=1}^{H} \gamma^{t-1}r_\D(s_t, a_t, \xi)$ and $R_\D(\tau, \xi):=\sum_{t=1}^{H} \gamma^{t-1}r_\D(s_t, a_t, \xi)$. Denote by $\mu(\tau; \theta, \phi)$ the trajectory distribution,  that the log probability of $\mu$ is given by 
\begin{align*}
 \log \mu(\tau;\theta,\phi) & = \sum_{t=1}^H (\log\pi_\D(a^t_\D|s^t;\theta +\eta \hat{\nabla}_{\theta} J_\D (\tau) )+\log\pi_\A(a^t_\A|s^t;\phi)+\log P(s^{t+1}|a^t_\D, a^t_\A,s^t)
\end{align*}
According to the policy gradient theorem, we have
\begin{align*}
    &\nabla_\phi \LL_\D(\theta, \phi, \xi)= \mathbb{E}_{\mu}[R_\D(\tau, \xi)\sum_{t=1}^H \nabla_\phi \log(\pi_\A(a_\A^t|s^t;\phi))],\\
    &\nabla_\phi \LL_\A(\theta, \phi, \xi )=\mathbb{E}_{\mu}[R_\A(\tau, \xi)\sum_{t=1}^H \nabla_\phi \log(\pi_\A(a_\A^t|s^t;\phi))].
\end{align*}
By SC \Cref{ass:sc}, when $\nabla_\phi \LL_\A(\theta, \phi, \xi ) = 0$, there exists $c < 0$, $d$, such that $\nabla_\phi \LL_\D(\theta, \phi, \xi) = \mathbb{E}_{\mu}[c R_\A(\tau, \xi)\sum_{t=1}^H \nabla_\phi \log(\pi_\A(a_\A^t|s^t;\phi))] + \mathbb{E}_{\mu}[ \sum_{t=1}^H \gamma^{t-1}d \sum_{t=1}^H \nabla_\phi \log(\pi_\A(a_\A^t|s^t;\phi))] = 0$. Hence $ \nabla_{\theta}V  = \E_{\xi \sim Q } [\nabla_{\theta} \LL_\D (\theta, \phi_\xi, \xi)]$.
  
    Third, $V(\theta)$ is also Lipschitz smooth. As we notice that, $\ell_\D$ is Lipschitz smooth since $\E_{\xi\sim Q}$ is a linear operator, we have,  
    \begin{equation*}
    \begin{aligned}
       & \quad \ \| \nabla_{\theta}V (\theta_1) - \nabla_{\theta}V (\theta_2) \|  \\
        &\leq \| \nabla_{\theta} \E_{\xi \sim Q}\LL_\D (\theta_1, \phi_1, \xi) - \nabla_{\theta} \E_{\xi \sim Q}\LL_\D (\theta_2, \phi_2, \xi)\| \\ 
       & = \| \nabla_{\theta}\ell_\D (\theta_1, \phi_1) - \nabla_{\theta}\ell_\D (\theta_2, \phi_1) + \nabla_{\theta}\ell_\D (\theta_2, \phi_1) - \nabla_{\theta}\ell_\D (\theta_2, \phi_2)\| \\
       & \leq \| \nabla_{\theta}\ell_\D (\theta_1, \phi_1) - \nabla_{\theta}\ell_\D (\theta_2, \phi_1)\| + \|\nabla_{\theta}\ell_\D(\theta_2, \phi_1) - \nabla_{\theta}\ell_\D(\theta_2, \phi_2)\| \\ 
        & \leq L_{11} \|\theta_1 - \theta_2 \| + L_{12} \|\phi_1 - \phi_2\|  \\
        & \leq (L_{11} +\frac{L_{12}L_{21}}{\mu}) \|\theta_1 - \theta_2 \|, 
        \end{aligned}
    \end{equation*}
     which implies the Lipschitz constant $L = L_{11} + \frac{L_{12}L_{21}}{\mu}$.
\end{proof}

It is impossible to present the convergence theory without the assistance of some standard assumptions in batch reinforcement learning, of which the justification can be found in~\citep{fallah2021convergence}. We also require some additional information about the parameter space and function structure. These assumptions are all stated in \Cref{ass:grad}.

\begin{assumption}\label{ass:grad}~
\begin{enumerate}
    \item[(a)] The following policy gradients are bounded, $\|\nabla_\phi \LL_\D (\theta, \phi, \xi) \| \leq G^2$, $ \| \LL_\A (\theta, \phi, \xi) \| \leq G^2$ for all $\theta, \phi \in \Theta \times \Phi$ and $\xi \in \Xi$. 
    \item[(b)] The policy gradient estimations are unbiased.
    \item[(c)] The variances for the stochastic gradients are bounded, i.e., for all $theta_\xi^t, \phi^t_\xi, \xi$,
 \begin{equation*}
    \E [   \| \hat{\nabla}_{\phi} J(\theta_\xi^t, \phi^t_\xi, \xi) - \nabla_{\phi} J(\theta_\xi^t, \phi^t_\xi, \xi) \|^2 ] \leq  \frac{\sigma^2}{N_b} .
\end{equation*}
    \item[(d)] The parameter space $\Theta$ has diameter $D_\Theta:= \sup_{\theta_1, \theta_2  \in \Theta} \|\theta_1 - \theta_2\|$; the initialization $\theta^0$ admits at most $D_V$ function gap, i.e., $D_V:= \max_{\theta \in \Theta} V(\theta) - V(\theta^0)$.
     \item[(e)] It holds that the parameters satisfy $  0 < \mu <  -cL_{22}$.
\end{enumerate}
\end{assumption}

Equipped with \Cref{ass:grad} we are able to unfold our main result \Cref{thm:main}, before which we show in \Cref{lemma:approxgradv} that $\phi^{*}_\xi$ can be efficiently approximated by the inner loop in the sense that $\nabla_\theta \E_{\xi \sim Q}\LL_\D (\theta^t, \phi^t_\xi(N_\A), \xi) \approx \nabla_\theta V(\theta^t)$, where $\phi^t_\xi(N_\A)$ is the last iterate output of the attacker policy.

\begin{lemma}\label{lemma:approxgradv}
 Under \Cref{ass:grad}, \ref{plass}, \ref{ass:sc}, and \ref{asslip}, let $\rho : = 1 + \frac{\mu }{c L_{22}} \in (0, 1)$, $ \bar{L} = \max \{ L_{11}, L_{12}, L_{22}, L_{21}, V_{\infty} \}$ where $V_{\infty} :=  \max\{ \max\|\nabla V(\theta)\|, 1 \}$. 
 For all $\varepsilon > 0$, if the attacker learning iteration $N_\A$ and batch size $N_b$ are large enough such that 
 \begin{equation*}
     \begin{aligned}
         N_\A & \geq \frac{1}{\log \rho^{-1}}\log \frac{32 D_V^2 (2V_{\infty} + LD_{\Theta})^4 \bar{L} |c|G^2    }{ L^2 \mu^2\varepsilon^4} \\
        N_b & \geq \frac{32 \mu L_{21}^2 D_V^2 ( 2 V_{\infty} +  L D_{\Theta} )^4}{ |c| L_{22}^2 \sigma^2  \bar{L} L\varepsilon^4} , 
     \end{aligned}
 \end{equation*}
 then, for $z_t := \nabla_{\theta} \E_{\xi \sim Q}\LL_\D (\theta^t, \phi^t_\xi(N_\A), \xi) - \nabla_{\theta} V(\theta^t)$, 
 \begin{equation*}
      \mathbb{E} [\|z_t\|]  \leq \frac{ L\varepsilon^2}{ 4D_V ( 2 V_{\infty} +  L D_{\Theta} )^2}  ,
 \end{equation*}
 and 
 \begin{equation*}
      \mathbb{E} [\| \nabla_{\phi} \LL_\A (\theta^t, \phi^t_\xi (N), \xi)\|] \leq\varepsilon. 
 \end{equation*}
\end{lemma}

\begin{proof}[Proof of \Cref{lemma:approxgradv}]
     Fixing a $\xi \in \Xi$, due to Lipschitz smoothness,
     \begin{align*}
        & \quad \ \LL_\D (\theta^t, \phi^t_\xi(N), \xi) - \LL_\D(\theta^t, \phi^t_\xi(N-1), \xi )
        \\ & \leq \langle \nabla_{\phi} \LL_\D (\theta^t, \phi^t_\xi(N-1), \xi), \phi^t_\xi(N) - \phi^t_\xi(N-1)\rangle + \frac{L_{22}}{2} \|\phi^t_\xi(N) - \phi^t_\xi(N-1)\|^2 .
     \end{align*}
     The inner loop updating rule ensures that when $\kappa_\A = \frac{1}{L_{21}}$, $ \phi^t_\xi(N) - \phi^t_\xi(N-1) = \frac{1}{L_{21}} \hat{\nabla}_{\phi} J_\A (\theta^t_\xi, \phi^t_\xi(N-1), \xi)$. 
     Plugging it into the inequality, we arrive at
     \begin{align*}
           & \ \quad \LL_\D (\theta^t, \phi^t_\xi(N), \xi) - \LL_\D(\theta^t, \phi^t_\xi(N-1), \xi ) 
          \\  & \leq   \frac{1}{L_{21}}\langle \nabla_{\phi} \LL_\D (\theta^t, \phi^t_\xi(N-1), \xi), \hat{\nabla}_{\phi}J_\A (\theta^t_\xi, \phi^t_\xi(N-1), \xi)\rangle + \frac{L_{22}}{2L^2_{21}} \|\hat{\nabla}_{\phi} J_\A (\theta^t_\xi, \phi^t_\xi(N-1), \xi)\|^2.
     \end{align*}
     Therefore, we let $(\mathcal{F}^t_n)_{ 0\leq n \leq N}$ be the filtration generated by $\sigma( \{ \phi^t_\xi (\tau) \}_{\xi \in \Xi}|\tau \leq n)$ and take conditional expectations on $\mathcal{F}^t_n$:
     \begin{align*}
         & \quad \E [ V(\theta^t) - \ell_\D(\theta^t, \phi^t(N)) |  \mathcal{F}^t_{N-1}] \leq V(\theta^t ) -  \ell_\D (\theta^t,\phi^t (N-1)) \\
         &  \E_\xi \left[ \frac{1}{L_{21}}\langle \nabla_{\phi} \LL_\D , \nabla_{\phi} J_\A (\theta^t_\xi, \phi^t_\xi(N-1), \xi)\rangle + \frac{L_{22}} {2L^2_{21}} \|\hat{\nabla}_{\phi} J_\A (\theta^t_\xi, \phi^t_\xi(N-1), \xi)\|^2 \right].
     \end{align*}
    By variance-bias decomposition, and \Cref{ass:grad} (b) and (c),
    \begin{align*}
    & \quad \ \E[ \| \hat{\nabla}_{\phi} J_\A (\theta^t_\xi, \phi^t_\xi(N-1), \xi) \|^2| \mathcal{F}^t_{N-1}] \\ & =  \E[ \| \hat{\nabla}_{\phi} J_\A (\theta^t_\xi, \phi^t_\xi(N-1), \xi) -  \nabla_{\phi} J_\A (\theta^t_\xi, \phi^t_\xi(N-1), \xi)  + \nabla_{\phi}J_\A (\theta^t_\xi, \phi^t_\xi(N-1), \xi) \|^2| \mathcal{F}^t_{N-1}] 
     \\ & =  \E[ \| (\hat{\nabla}_{\phi} -  \nabla_{\phi}) J_\A (\theta^t_\xi, \phi^t_\xi(N-1), \xi)\|^2 | \mathcal{F}^t_{N-1} ]  +  \E [\| \nabla_{\phi} J_\A (\theta^t_\xi, \phi^t_\xi(N-1), \xi) \|^2 | \mathcal{F}^t_{N-1}] 
      \\ & \quad  + \E [ 2 \langle (\hat{\nabla}_{\phi} -  \nabla_{\phi}) J_\A (\theta^t_\xi, \phi^t_\xi(N-1), \xi) ,  \nabla_{\phi} J_\A (\theta^t_\xi, \phi^t_\xi(N-1), \xi) \rangle |\mathcal{F}^t_{N-1} ]
     \\ & \leq \frac{\sigma^2}{N_b} + \| \nabla_{\phi} J_\A (\theta^t_\xi, \phi^t_\xi(N-1), \xi) \|^2    .
    \end{align*}
    Applying the PL condition (\Cref{plass}), and \Cref{ass:grad} (a) we obtain
      \begin{align*}
        & \quad  \E[ V(\theta^t) - \ell_\D(\theta, \phi^t(N))| \phi^{N-1}]
        - V(\theta^t) -  \ell_\D(\theta, \phi^t(N-1))  
        \\ & \leq  \E_\xi \left[\frac{1}{ L_{21}}\langle \nabla_{\phi} \LL_\D , \nabla_{\phi}\LL_\A (\theta^t, \phi^t_\xi(N-1), \xi)\rangle  + \frac{L_{22}}{2 L_{21}^2} ( \frac{\sigma^2}{ N_b}+ \|\nabla_{\phi}\LL_\A (\theta^t, \phi^t_\xi(N-1), \xi) \|^2  ) \right]
        \\ & =  \E_\xi \left[  -\frac{1}{2L_{22}} \| \nabla_{\phi } \LL_\D \|^2 +  \frac{1}{2 L_{22}}\|\nabla_{\phi} (\LL_\D  +  \frac{L_{22}}{L_{21}} \LL_\A) (\theta^t, \phi^t_\xi(N-1), \xi) \|^2 + \frac{L_{22}\sigma^2 }{2 L_{21}^2 N_b} \right]
        \\ & \leq  \frac{\mu}{c L_{21}} ( \max_{\phi} \ell_\D (\theta^t, \phi)  - \ell_\D (\theta^t, \phi^t (N-1)) ) + \frac{L_{22} \sigma^2  }{2 L_{21}^2 N_b} , 
       \end{align*}
    
   rearranging the terms yields 
    \begin{align*}
         \E [ V(\theta^t) - \ell_\D(\theta^t,\phi^t (N)) | \mathcal{F}^t_n] \leq \rho ( V(\theta^t) - \ell_\D(\theta^t,\phi^t (N-1)))  + 
 \frac{L_{22}\sigma^2  }{2 L_{21}^2 N_b} , 
    \end{align*}
    where we use the fact that $ - \max_{\phi} \ell_\D (\theta^t, \phi) \leq -V(\theta^t)$.
    Telescoping the inequalities from $\tau = 0$ to $\tau = N$, we arrive at 
    \begin{align*}
         \E [ V(\theta^t) - \ell_\D(\theta^t, \phi^t(N))] \leq \rho^N (V(\theta^t) - \ell_\D(\theta^t, \phi^t(0))) + \frac{1 - \rho^N }{1 - \rho} \left(\frac{L_{22}\sigma^2  }{2 L_{21}^2 N_b}\right).
    \end{align*}   

PL-condition implies quadratic growth, we also know that $ V(\theta^t) -  \ell_\D (\theta^t, \phi^t(N)) \leq \E_\xi \frac{1}{2\mu} \| \nabla_{\phi} \LL_\D (\theta^t, \phi^t_\xi(N), \xi)\|^2 \leq \frac{1}{2\mu} G^2 $, by  \Cref{ass:sc}, 
 \begin{align*}
    \|   \phi^*_\xi (\theta^t) - \phi^t_\xi (N) \|^2  &\leq \frac{2}{\mu} (\LL_\A(\theta^t, \phi^*_\xi, \xi) - \LL_\A(\theta^t, \phi^t_\xi (N), \xi)) 
    \\ & \leq  \frac{2|c| }{\mu} \big\vert \LL_\D (\theta^t, \phi^*_\xi, \xi) - \LL_\D (\theta^t, \phi^t_\xi(N), \xi ) \big\vert
 \end{align*}
 Hence, with Jensen inequality and choice of $N_\A$ and $N_b$, 
     \begin{align*}
         \E [ \|z_t \|] & = \E [ \| \nabla_{\theta} V(\theta^t) -  \E_\xi \nabla_{\theta} \LL_\D (\theta^t, \phi^t_\xi (N_\A), \xi )\|] \\
         & \leq L_{12} \E [ \|\phi^t_\xi (N_\A) -  \phi^*_\xi \|]  \\
        & \leq  L_{12} \sqrt{\frac{2 |c|}{\mu}\E [ V(\theta^t) - \ell_\D(\theta^t, \phi^t(N_\A))] } \\
        & \leq   L_{12} \sqrt{\frac{ |c|}{\mu^2} \rho^{N_\A} G^2 + (1 - \rho^{N_\A}) \frac{|c| L_{22}^2 \sigma^2 }{\mu L_{21}^2 N_b}} .  
     \end{align*}
Now we adjust the size of $N_\A$ and $N_b$ to make $\E [ \|z_t \|]$ small enough, to this end, we set
\begin{align*}
  \rho^{N_\A} \frac{|c| G^2}{ \mu^2} & \leq \frac{\varepsilon^4 L^2  }{32 D_V^2 (2V_{\infty} + L D_{\Theta})^4 \bar{L} } \\
   \frac{|c| L_{22}^2 \sigma^2 }{  L_{21}^2 N_b} & \leq \frac{\varepsilon^4 L^2 \mu^2 }{32 D_V^2 (2V_{\infty} + LD_{\Theta})^4 \bar{L} }, 
\end{align*} 
which further indicates that 
\begin{align*}
    N_\A &   \geq \frac{1}{\log \rho^{-1}}\log \frac{32 D_V^2 (2V_{\infty} + LD_{\Theta})^4 \bar{L} |c|G^2    }{ L^2 \mu^2\varepsilon^4} \\
    N_b &  \geq \frac{32 \mu L_{21}^2 D_V^2 ( 2 V_{\infty} +  L D_{\Theta} )^4}{ |c| L_{22}^2 \sigma^2  \bar{L} L\varepsilon^4} .
\end{align*}
In the setting above, it is not hard to verify that
\begin{align*}
 \E [\|z_t \|] \leq \frac{ L\varepsilon^2}{ 4D_V ( 2 V_{\infty} +  L D_{\Theta} )^2} \leq\varepsilon. 
\end{align*}
Also note that $\| \nabla_{\phi} \LL_\A (\theta^t, \phi^t_\xi(N_\A), \xi) \| = \| \nabla_{\phi}  \LL_\A (\theta^t, \phi^t_\xi(N_\A), \xi) - \nabla_{\phi} \LL_\A (\theta^t, \phi^*_\xi, \xi)\|$, given the proper choice  of $N_\A$ and $N_b$, one has
      \begin{align*}
           & \quad \E \| \nabla_{\phi}  \LL_\A (\theta^t, \phi^t_\xi(N_\A), \xi) - \nabla_{\phi} \LL_\A (\theta^t, \phi^*_\xi, \xi)\|
          \\ & \leq  L_{21} \E [ \|\phi^t_\xi (N_\A) -  \phi^*_\xi \| ]  \leq  \frac{ L\varepsilon^2}{ 4D_V ( 2 V_{\infty} +  L D_{\Theta} )^2}  \leq \varepsilon ,
      \end{align*}
which indicates the $\xi$-wise inner loop stability.
\end{proof}

Now we are ready to provide the convergence guarantee of the first-order outer loop. 

\begin{theorem}
 Under \Cref{ass:grad}, \Cref{ass:sc}, and \Cref{asslip}, let the stepsizes be, $\kappa_\A = \frac{1}{L_{22}}$, $\kappa_\D = \frac{1}{L}$, if $N_\D, N_\A,$ and $N_b$ are large enough,
 \begin{align*}
     N_\D \geq N_\D(\varepsilon) \sim  \mathcal{O}(\varepsilon^{-2}) \quad N_\A  \geq N_\A (\varepsilon) \sim \mathcal{O} ( \log\varepsilon^{-1}), \quad N_b \geq N_b (\varepsilon)\sim \mathcal{O} (\varepsilon^{-4})
 \end{align*}
 then there exists $t \in \mathbb{N}$ such that $(\theta^t, \{\phi^{t}_\xi(N_\A) \}_{\xi \in \Xi})$ is $\varepsilon$-meta-FOSE.
\end{theorem}

\begin{proof} 
According to the update rule of the outer loop, (here we omit the projection analysis for simplicity)
\begin{align*} 
 \theta^{t+1} - \theta^t = \frac{1}{L}\hat{\nabla}_{\theta} \ell_\D (\theta^t, \phi^t(N_\A)), 
\end{align*}
one has, due to unbiasedness assumption, let $(\mathcal{F}_t)_{ 0 \leq t\leq N_\D}$ be the filtration generated by $\sigma(\theta^t| k \leq t)$ 
\begin{align*}
     \E[ \langle  \nabla_{\theta} \ell_\D (\theta^t, \phi^t(N_\A)),  \theta^{t+1} - \theta^t \rangle | \mathcal{F}_t ] & =  \frac{1}{L}\E[ \| 
 \nabla_{\theta} \ell_\D(\theta^t, \phi^t(N_\A))\|^2 |\mathcal{F}_t] 
 \\ &  = L  \E \| \theta^{t+1} - \theta^t\|^2 | \mathcal{F}_t ], 
\end{align*}
which leads to
\begin{align*}
    \E[ \langle \nabla_{\theta} \ell_\D(\theta^t, \phi^*) , \theta^{t+1} - \theta^t \rangle | \mathcal{F}_t ] & = \E[ \langle z_t, \theta^t - \theta^{t+1}\rangle| \mathcal{F}_t ] + L \E[ \| \theta^{t+1} - \theta^t\|^2\|]  . 
\end{align*}
 Since $V(\cdot)$ is $L$-Lipschitz smooth,
 \begin{equation} \label{telev}
 \begin{aligned}
    \E[ V( \theta^{t}) -  V(\theta^{t+1} )] & \leq  \E[ \langle  \nabla_{\theta} V(\theta^t ), \theta^{t} - \theta^{t+1} \rangle] + \frac{L}{2} \E[ \| \theta^{t+1} - \theta^t\|^2 ]
   \\  \leq  \E[\langle z_t, \theta^{t+1}  &- \theta^{t}\rangle] - \E[\langle \nabla_{\theta} \ell_\D(\theta^t, \phi^t(N_\A) ),  \theta^{t+1} - \theta^t\rangle] + \frac{L}{2} \E[ \| \theta^{t+1} - \theta^t\|^2 ] 
    \\  \leq \E[\langle z_t, \theta^{t+1}  &- \theta^{t}\rangle] - \frac{L}{2} \E[ \| \theta^{t+1} - \theta^t\|^2 ] .
 \end{aligned}
 \end{equation}

Fixing a $\theta \in \Theta$, let $e_t :=  \langle \nabla_{\theta} \ell_\D(\theta^t, \phi^t(N_\A) ), \theta - \theta^t \rangle$, we have
\begin{equation} \label{boundet}
 \begin{aligned}
     \E[e_t | \mathcal{F}_t ] & =  L \E[  \langle \theta^{t+1} - \theta^t   , \theta - \theta^{t} \rangle | \mathcal{F}_t] 
     \\ &  =  \E[ \langle  \nabla_{\theta} \ell_\D (\theta^t, \phi^t(N_\A)) - \nabla_{\theta} V(\theta^t), \theta^{t+1} - \theta^{t}  \rangle + \langle  \nabla_{\theta} V(\theta^t),  \theta^{t+1} - \theta^t \rangle]
    \\ & \quad  + L \E[  \langle \theta^{t+1} - \theta^t,  \theta - \theta^{t+1} \rangle ]
    \\ & \leq   \E [ (  \|z_t\| + V_{\infty} + LD_{\Theta}) \| \theta^{t+1} - \theta^t \|]
 \end{aligned}
 \end{equation}

  By the choice of $N_b$, we have, since $V_{\infty} = \max\{ \max_{\theta} \|\nabla V(\theta)\|, 1\}$,
 \begin{align*}
      \E [ \|z_t\|] \leq L_{12} \E [ \| \phi^N - \phi^*\|] \leq \frac{ L\varepsilon^2}{ 4D_V ( 2 V_{\infty} +  L D_{\Theta} )} \leq V_{\infty} .
 \end{align*}
Thus, the relation \eqref{boundet} can be reduced to 
\begin{equation*}
\E [ e_t] \leq (2 V_{\infty}+ LD_{\Theta} ) \E [\|\theta^{t+1} - \theta^t\|]. 
\end{equation*}
Telescoping \eqref{telev} yields 
\begin{equation*}
- D_V \leq \E[ V(\theta^0) - V(\theta^{N_\D})] \leq  D_{\Theta} \sum_{t=0}^{T-1} \E [ \| z_t \|] -  \frac{L}{2(2 V_{\infty} + L D_{\Theta})^2} \E[\sum_{t=0}^{T-1} \E[ e^2_t | \mathcal{F}_t ].
\end{equation*}

Thus, setting $N_\D \geq \frac{ 4D_V (2 V_{\infty} + LD_{\Theta})^2}{L\varepsilon^2 }$, and then by Lemma \ref{lemma:attackstablize}, we obtain that, 
\begin{equation*}
    \frac{1}{N_\D } \sum_{t=0}^{N_\D -1} \E [e^2_t ]  \leq \frac{\varepsilon^2 }{2}  + \frac{ 2 D_V(2V_{\infty}+ LD_{\Theta})^2  }{L N_\D }  \leq \varepsilon^2
\end{equation*}
which implies there exists $t \in \{0, \ldots, N_\D - 1\}$ such that $\E [ e^2_t] \leq\varepsilon^2$.


\end{proof}

\ignore{

\section{Methodology}
\label{app:methodology}
\subsection{Simulated Environment}
\label{subsec: simulated env}

In the process of emulating transitions and reward functions within BSMG, our starting point is the presumption that the defender invariably contemplates the most adverse situation. This judgement is grounded on a cursory understanding of the quantity of malevolent clients under each attacker's control and the uneven data distribution levels among the clients. As an illustration, if the real ratio of malevolent devices swings between $10\%$ and $40\%$, the defender will presume the worst-case of $40\%$.
Secondly, replicating client actions, particularly local training, demands an extensive data volume, which is often not readily available. Thus, we make the assumption that a small base dataset is accessible to the server (a typical supposition in Federated Learning) and we utilize data augmentation techniques to produce a larger number of data samples. The subsequent data is then employed to train a conditional GAN model~\cite{mirza2014conditional} for MNIST and a diffusion model~\cite{sohl2015deep} for CIFAR-10, facilitating the generation of adequate data to emulate local training within the simulated setup.

In reality, the defender (the server, in this case) lacks knowledge regarding the triggers or target labels of the backdoor attacker. In an attempt to replicate the actions of a backdoor attacker, we utilize several GAN-based attack models to create various trigger distributions, each focusing on a single label. These are then perceived as varying attack forms within the simulation. Given that the defender is unaware of the poison ratio and targeted label within the attacker's tainted dataset, we adjust the defender's reward function to reflect this.
As for the reward function, it's defined as $r_{\mathcal{D}}^t=-\mathbb{E}[F''(\widehat{w}^{t+1}_g)]$, $F''(w):=\lambda' F(w,U)-(1-\lambda') \min_{c\in C} [\frac{1}{|U'|}\sum_{j=1}^{|U'|}\ell(w,(\hat{x}_i^j,c))]\geq \lambda' F(w,U)-(1-\lambda') [\frac{1}{|U'|}\sum_{j=1}^{|U'|}\ell(w,(\hat{x}_i^j,c^*))]$, where $c^*$ is the truly targeted label, and $\lambda'\in [0,1]$ is a measure of the tradeoff between the main task and the backdoor task. We take all the data in $U'$ to be poisoned, to give a rough approximation of the real attack objective $\lambda F(w,U)+(1-\lambda)F(w,U')$ with another $\lambda$.

\subsection{RL-based Attacks and Defenses}
\label{subsec: RL defense}
In line with the BSMG model, the intuitive approach would be to employ $w_g^t$ or $(w_g^t,\mathbf{I}^t)$ as the state, and ${\widetilde{g}k^t}{k=1}^{M_1+ M_2}$ or $w_g^{t+1}$ as the actions for the attacker and the defender respectively, provided the federated learning model is small. However, when training a high-dimensional model (i.e., a large neural network) through federated learning, the original state/action space could result in an overwhelmingly large search space, which would be prohibitively demanding in terms of training time and memory space.
In order to mimic an adaptive model poisoning attack, we employ the RL-based attack~\cite{li2022learning}. Similarly, an RL-based local search~\cite{li2023learning} is used to simulate an adaptive backdoor attack. Both of these methods possess 3-dimensional real action spaces subsequent to action comparison. We further streamline the process by insisting that all malicious devices under the control of the same attacker perform the same action.
To condense the state space, we reduce $w_g^t$ to only its final two hidden layers for both the attacker and defender. For the attacker, the identity vector is minimized to represent the number of malicious clients sampled at round $t$. As a universal defense against various types of attacks, we utilize reinforcement learning for hyperparameter optimization.

In the context of defense against \textbf{untargeted model poisoning attacks}, we look at methods like coordinate-wise trimmed mean, which involves a trimmed rate $\beta=[0,\frac{1}{2})$ (dimension-wise), clipping that comes with a norm bound $\alpha$ (magnitude), and FoolsGold that uses a cosine similarity threshold $\kappa$ (direction). These are all training stage defenses reliant on aggregation.
In the fight against \textbf{backdoor attacks}, we introduce random noise at a level of $\delta$ and clip the model updates with a norm bound $\alpha$ as part of the training stage defense. In the post-training stage, we consider Neuron Clipping, which involves a clip range of $\varepsilon$, or Pruning, which has a pruning mask rate of $\sigma$.
It's worth noting that each of these defensive measures can be swapped out with other algorithms. The unique element of our approach lies not in the use of fixed and hand-crafted hyperparameters, for example $a^t_{1}:=(\beta, \alpha, \kappa)$ in the untargeted defense and $a^t_{2}:=(\delta, \alpha,\varepsilon \slash \sigma)$ in the backdoor defense, but in our use of reinforcement learning (RL) to optimize these hyperparameters. Instead of concurrently searching for all parameters, the two defenses, which span a total of six dimensions, are optimized separately.

\section{Experiment Setup}
\label{app:setup}
\paragraph{Data Sources.} Our analysis utilizes two distinct datasets: the MNIST \citep{lecun1998gradient} and CIFAR-10 \citep{krizhevsky2009learning}. We implement standard $i.i.d.$ local data distributions, splitting each dataset randomly into $n$ equally sized subsets for training. The MNIST dataset is made up of 60,000 training and 10,000 testing examples, each being a 28x28 grayscale image linked to one of 10 classifications. The CIFAR-10 dataset, on the other hand, is comprised of 60,000 color images spanning 10 different categories, split into 50,000 training and 10,000 testing examples.
When examining the {\em non-i.i.d.} context (refer to Figure~\ref{fig:ablation}(d)), we utilize the approach proposed by~\citep{fang2020local} to measure data heterogeneity. Both MNIST and CIFAR-10 datasets are split into $C=10$ groups of workers. Non-i.i.d. federated learning is modeled by allocating a training instance with label $c$ to the corresponding $c$-th group with probability $q$, and to all groups with a probability of $1-q$. Greater heterogeneity is indicated by a higher value of $q$.


\paragraph{Federated Learning Configuration.} Our default parameters for the federated learning (FL) setup include a local minibatch size of 128, a single local iteration, a learning rate of 0.05, 100 workers, 5 backdoor attackers, 20 untargeted model poisoning attackers, a subsampling rate of $10\%$, and 500 and 1000 FL training rounds for MNIST and CIFAR-10, respectively. For the MNIST dataset, we employ a neural network classifier incorporating convolutional filter layers of dimensions 8x8, 6x6, and 5x5 with ReLU activations, succeeded by a fully connected layer and a softmax output. In the case of the CIFAR-10 dataset, we use the ResNet-18 model~\citep{he2016deep}.
The FL model is implemented using PyTorch~\citep{paszke2019pytorch} and all experiments are executed on a homogeneous 2.30GHz Linux system, equipped with a 16GB NVIDIA Tesla P100 GPU.
The standard loss function and optimizer used are cross-entropy loss and stochastic gradient descent (SGD), respectively. Apart from the experiments showcased in Figures~\ref{fig:ablation}(c) and~\ref{fig:ablation}(d), we maintain the initial model and random seeds of subsampling consistent for balanced comparisons.


\paragraph{Benchmark Comparisons.} We assess the efficacy of our defensive approach against a range of contemporary attack strategies. These include both non-adaptive and adaptive untargeted model poisoning attacks such as IPM~\cite{xie2020fall}, LMP~\cite{fang2020local}, and RL~\cite{li2022learning}, in addition to backdoor attacks like BFL~\cite{bagdasaryan2020backdoor} and BRL~\cite{li2023learning} (with a tradeoff parameter $\lambda=0.5$), and a combination of both types of attacks. To establish the potency of our defense, we draw comparisons with a number of robust defense mechanisms. The selected benchmarks consist of defenses activated during the training phase, like Krum~\cite{blanchard2017machine}, Clipping Median~\cite{yin2018byzantine,sun2019can,li2022learning} (with a norm bound set to 1), and FLTrust~\cite{cao2020fltrust} with 100 root data samples and a bias of $q=0.5$, in addition to post-training defenses such as Neuron Clipping~\cite{wang2022universal} and Pruning~\cite{wu2020mitigating}. We adopt the original clipping thresholds of 7 as outlined in~\cite{wang2022universal} and designate the default pruning number as 256.


\paragraph{Reinforcement Learning Configuration.} For our RL-based defense strategy, we employ the Twin Delayed DDPG (TD3)\citep{fujimoto2018addressing} algorithm, given that both the action and state spaces are continuous. This algorithm is employed to independently train the policies for untargeted defense and backdoor defense. The simulation environment is implemented using OpenAI Gym\citep{1606.01540} and the TD3 is set up using OpenAI Stable Baseline3~\citep{stable-baselines3}.
The parameters for RL training are specified as follows: the number of FL rounds is set to 300, the policy learning rate is 0.001, the policy model is based on the MultiInput Policy, the batch size is set to 256, and a $\gamma$ value of 0.99 is used for updating the target networks.
A default $\lambda$ value of 0.5 is utilized when computing the backdoor rewards.


\begin{figure*}[!t]
\centering
    \includegraphics[width=0.7\textwidth]{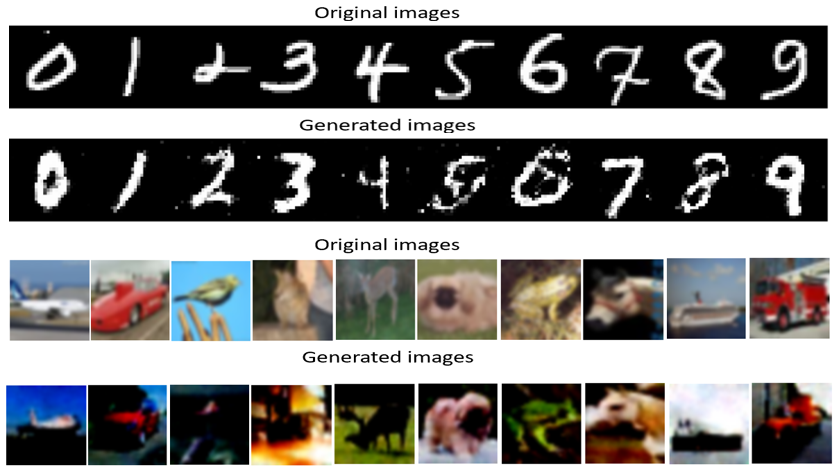}
  \caption{Self-generated MNIST images using conditional GAN~\cite{mirza2014conditional} (second row) and CIFAR-10 images using a diffusion model~\cite{sohl2015deep} (fourth row). 
  }
  \label{fig:generated}
\end{figure*}

\begin{figure*}[!t]
\centering
    \includegraphics[width=0.9\textwidth]{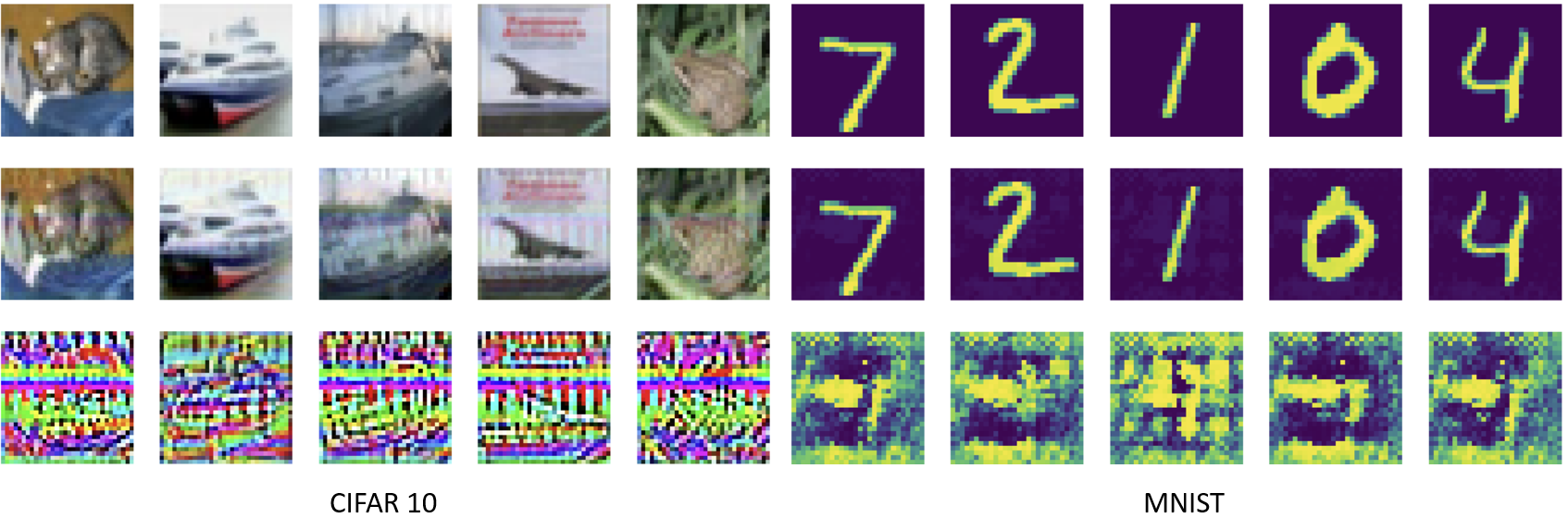}
  \caption{Generated backdoor triggers using GAN-based models~\cite{doan2021lira}. Original image (first row). Backdoor image (second row). Residual (third row).}
  \label{fig:gan_trigger}
\end{figure*}

\begin{figure*}
    \centering
    \includegraphics[width=1\textwidth]{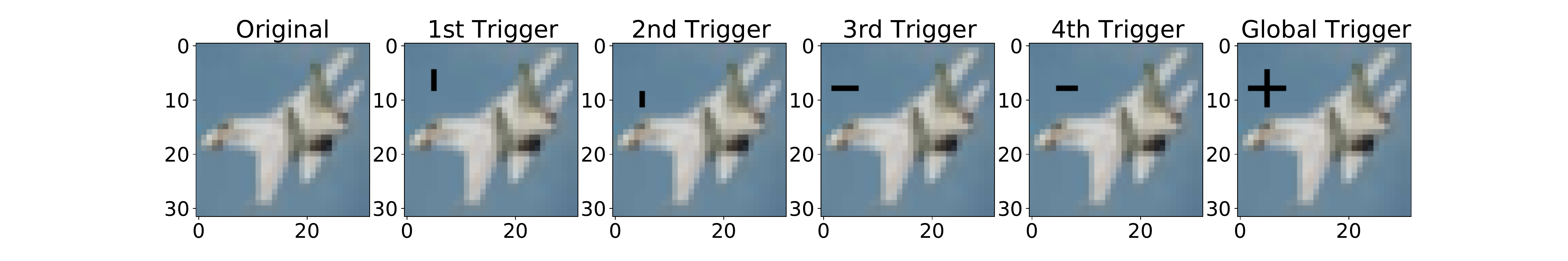}
    \caption{\small{CIFAR-10 fixed backdoor trigger patterns.
    The global trigger is considered the default poison pattern and is used for online adaptation stage backdoor accuracy evaluation. The sub-triggers are used by pre-training only.}}
    \label{fig:cifar10_dba}
\end{figure*}

\paragraph{Meta-learning Configuration.}
The potential attack sets, or attack domains, are established as follows:
For meta-RL, we include IPM~\citep{xie2020fall}, LMP~\citep{fang2020local}, and EB~\citep{bhagoji2019analyzing} as the three viable attack methods. For the meta-SG against the untargeted model poisoning attack, we consider RL-based attacks~\citep{li2022learning}, trained to combat Krum~\citep{blanchard2017machine} and Clipping Median~\citep{li2022learning, yin2018byzantine, sun2019can}, as initial attacks. As for meta-SG against the backdoor attack, RL-based backdoor attacks~\cite{li2023learning} trained against Norm-bounding~\citep{sun2019can} and Neuron Clipping~\citep{wang2022universal} (Pruning~\cite{wu2020mitigating}) are considered as initial attacks. For meta-SG against mixed types of attacks, we include both RL-based attacks~\citep{li2022learning} and RL-based backdoor attacks~\cite{li2023learning} described above as initial attacks.

During the pre-training phase, we set the number of iterations $T$ at 100. In each iteration, we randomly sample $K=10$ attacks from the attack type domain (refer to Algorithm~\ref{algo:meta-rl} and Algorithm~\ref{algo:meta-sl}). For each attack, a trajectory of length $H=200$ is generated for MNIST ($H=500$ for CIFAR-10), and we update the policies of both the attacker and the defender for 10 steps using TD3 (that is, $l=N_{\A}=N_{\D}=10$). In the online adaptation phase, the meta-policy undergoes adaptation for $100$ steps utilizing TD3 with $T=10$, $H=100$ for MNIST ($H=200$ for CIFAR-10), and $l=10$. The remaining parameters are defined as follows: the single task step size $\kappa=\kappa_{\A}=\kappa_{\D}= 0.001$, meta-optimization step size is set to 1, and the adaptation step size is 0.01.

\paragraph{Self-generated Dataset.}
We start by recognizing that the server possesses a limited initial dataset (specifically, 100 samples with $q=0.1$), a common scenario in the FL ~\citep{cao2020fltrust}. To emulate a training set with 60,000 images (applicable to both MNIST and CIFAR-10) for FL, this sparse dataset is expanded through several methods including normalization, random rotation, and color jittering, yielding a more extensive and diverse dataset. This generated dataset will be subsequently utilized as input for generative models.

For the MNIST dataset, a Conditional Generative Adversarial Network (cGAN) model~\cite{mirza2014conditional,odena2017conditional} is trained using the expanded dataset. This model, built on the codebase presented in~\cite{Pytorch-cGAN}, includes two primary components - a generator and a discriminator, both of which are neural networks. More specifically, a dataset consisting of 5,000 augmented data samples is used to train the cGAN. The network parameters are retained as default, and the training epoch is set to 100.

In the case of CIFAR-10, a diffusion model as implemented in~\cite{cifar-diffusion} is utilized. This model integrates several contemporary techniques, encompassing a Denoising Diffusion Probabilistic Model (DDPM)\cite{ho2020denoising}, deterministic sampling in the style of DDIM\cite{song2020denoising}, and continuous timesteps parametrized by the log SNR at each timestep~\cite{kingma2021variational} to facilitate varying noise schedules during sampling. The model also applies the `v' objective, which stems from Progressive Distillation for Fast Sampling of Diffusion Models \cite{salimans2022progressive}, improving the conditioning of denoised images at high noise levels. During the training phase, a dataset containing 50,000 augmented data samples is used as input to train this model. The parameters are retained as default, and the training epoch is set to 30.

\paragraph{Backdoor Attacks.}

Our evaluation of backdoor attacks takes into account the trigger patterns demonstrated in Figure~\ref{fig:gan_trigger} and Figure~\ref{fig:cifar10_dba}. In the scenario where triggers are generated via GAN (see Figure~\ref{fig:gan_trigger}), the objective is to categorize images from various classes into a single target class, also known as the all-to-one classification. In the context of predetermined patterns (as seen in Figure~\ref{fig:cifar10_dba}), the aim is to categorize images from the airplane class into the truck class, a one-to-one classification scenario. In both situations, the poison ratio is set at 0.5 by default. The global trigger depicted in Figure~\ref{fig:cifar10_dba} serves as the standard poison pattern and is utilized during the online adaptation phase for backdoor accuracy evaluation.


\section{Experiment Results}
\label{app:extended}
\begin{figure*}[t]
 	\vspace{-5pt}
 	\centering
 		\subfloat[]{%
 		\includegraphics[width=0.25\textwidth]{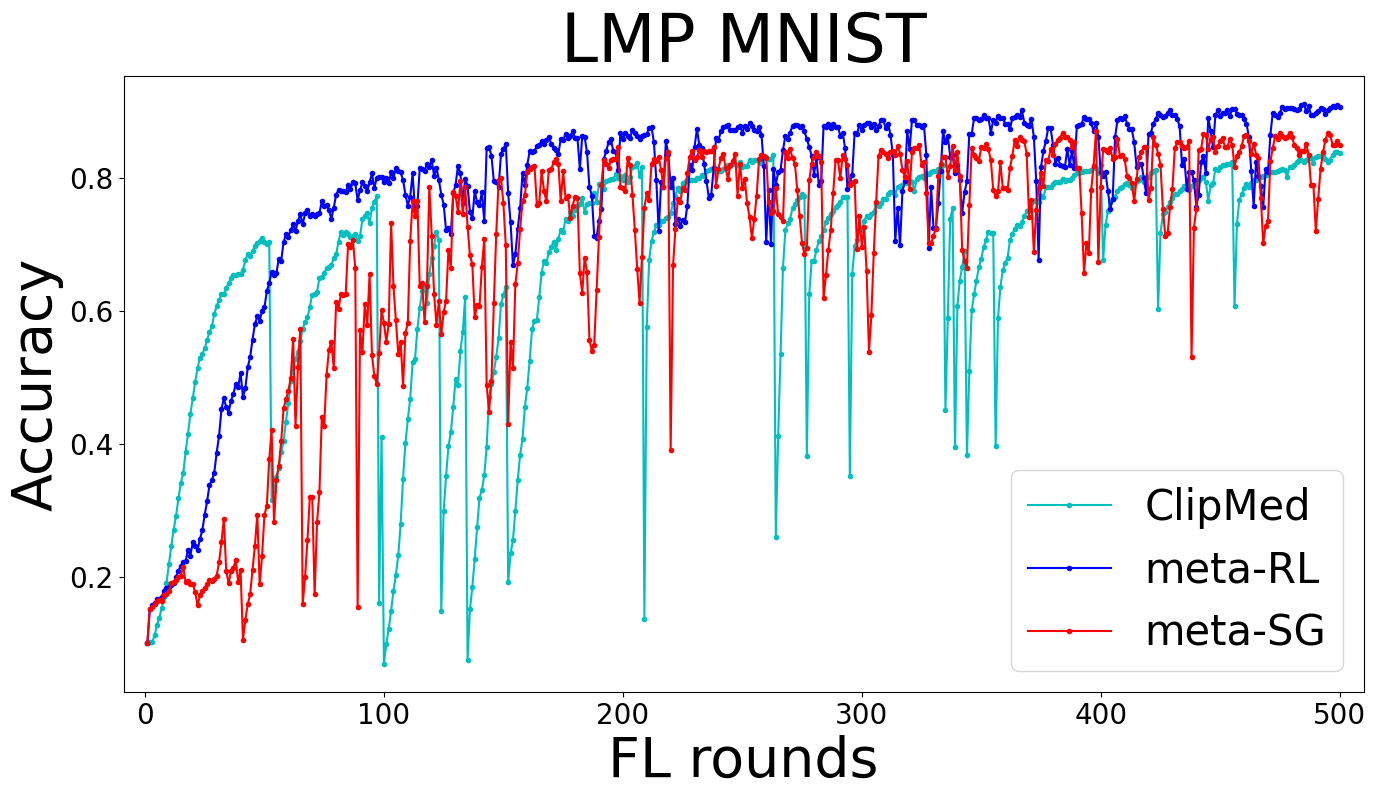}%
 	}
 	\subfloat[]{%
 		\includegraphics[width=0.25\textwidth]{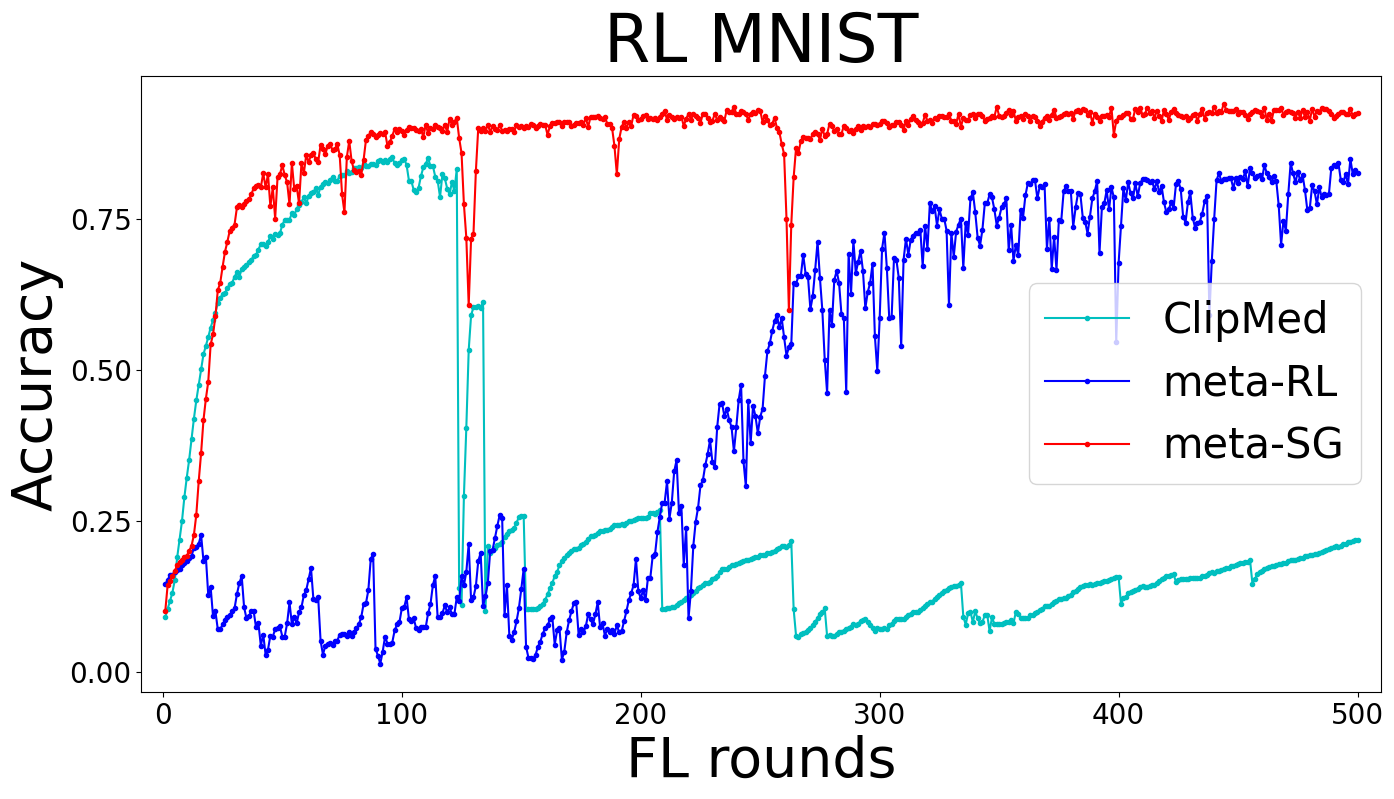}%
 	}
 	\subfloat[]{%
 		\includegraphics[width=0.25\textwidth]{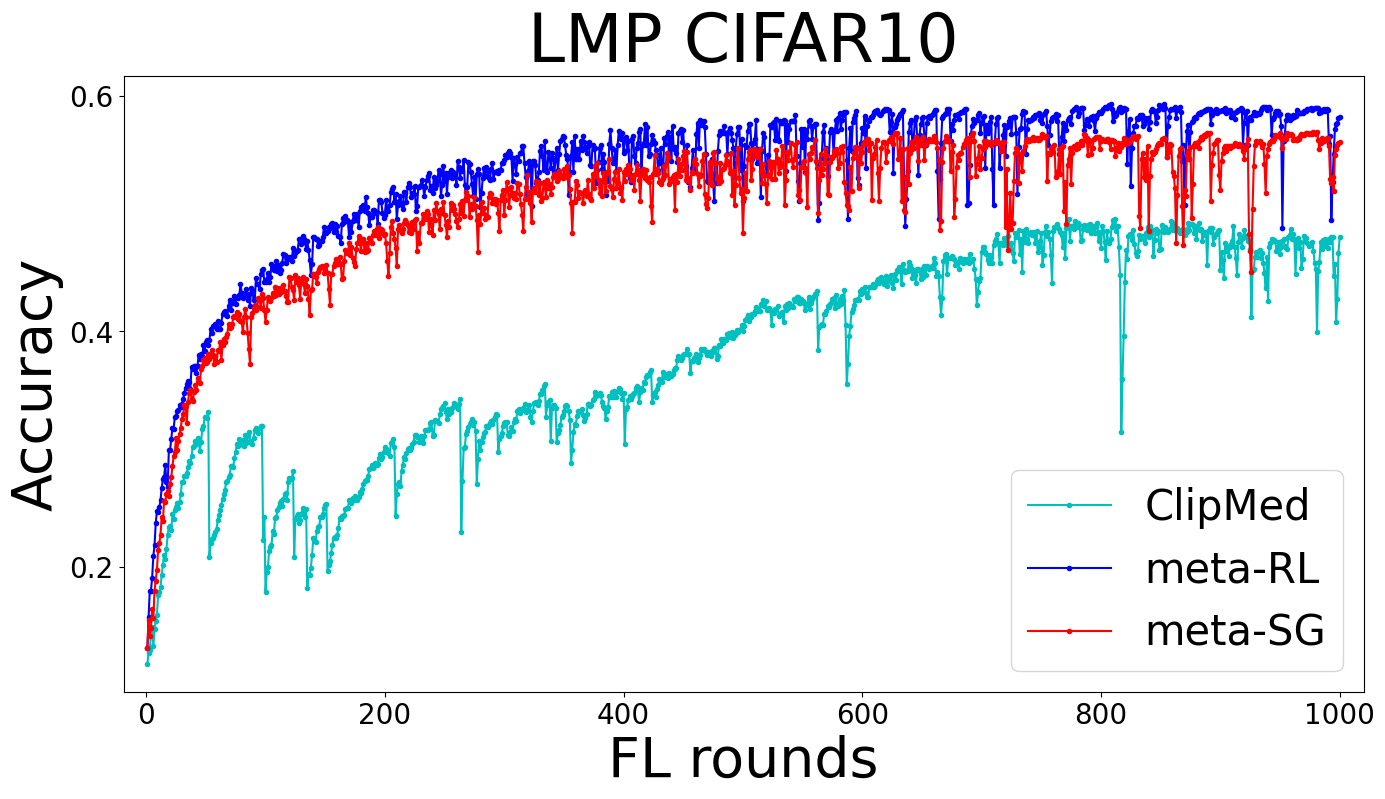}%
 	}
 	\subfloat[]{%
 		\includegraphics[width=0.25\textwidth]{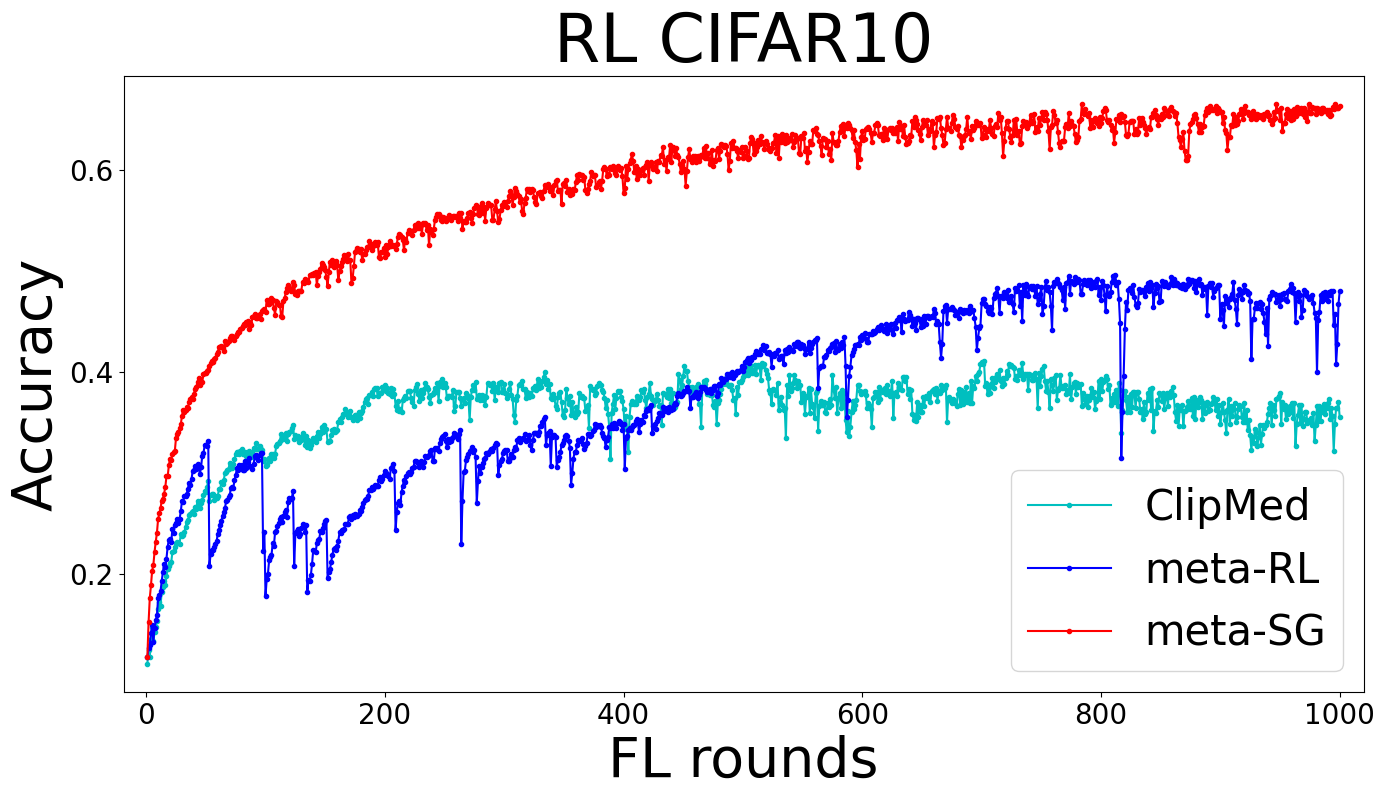}%
 	}

 	 \caption{\small Comparison of defenses against untargeted model              poisoning attacks, namely LMP and RL, applied on MNIST and CIFAR-10         datasets. All parameters set as default.}
 	\vspace{-5pt}
 	\label{fig:untargeted}
\end{figure*}

\begin{figure}[t]
    \centering
  \begin{subfigure}{0.24\textwidth}
      \centering
          \includegraphics[width=\textwidth]{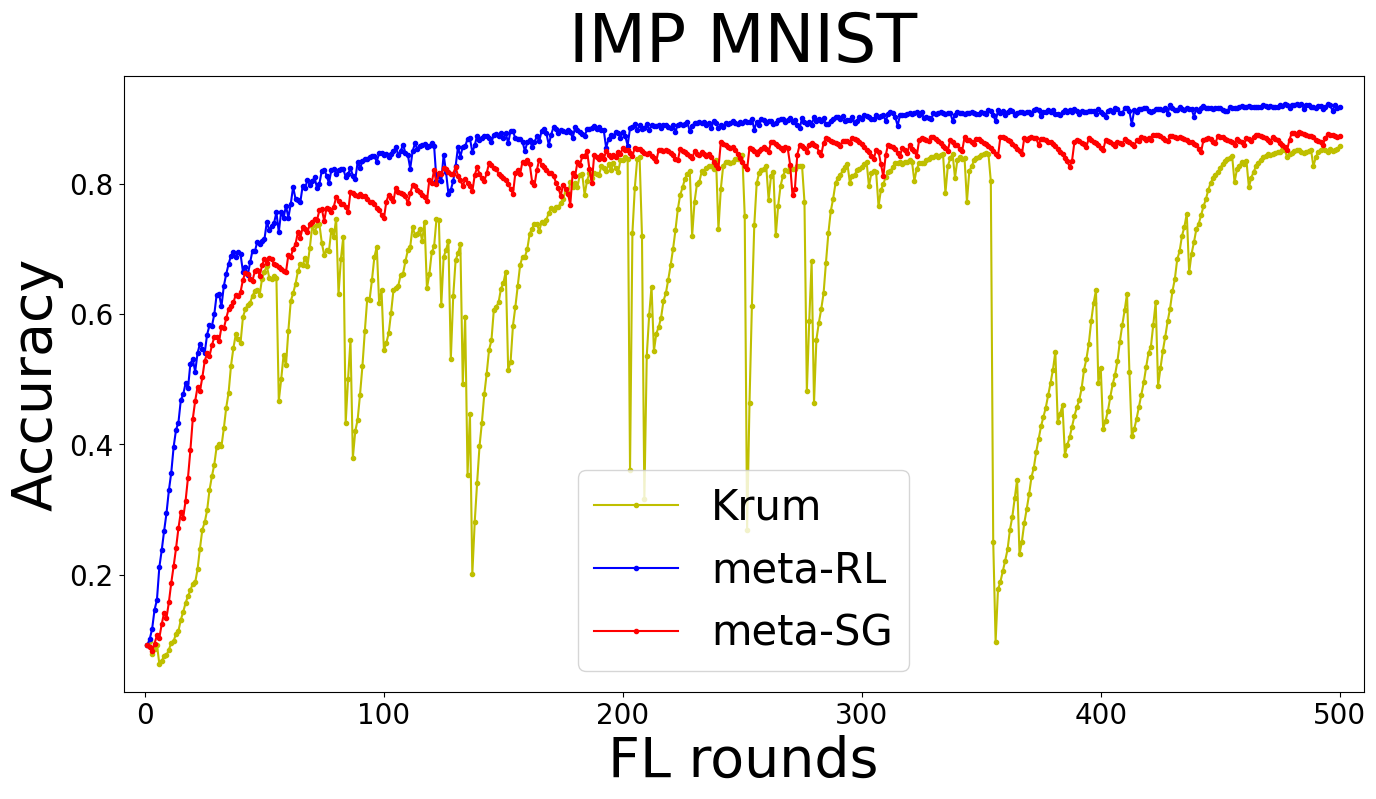}
          \caption{}
  \end{subfigure}
  \hfill
    \begin{subfigure}{0.24\textwidth}
      \centering
          \includegraphics[width=\textwidth]{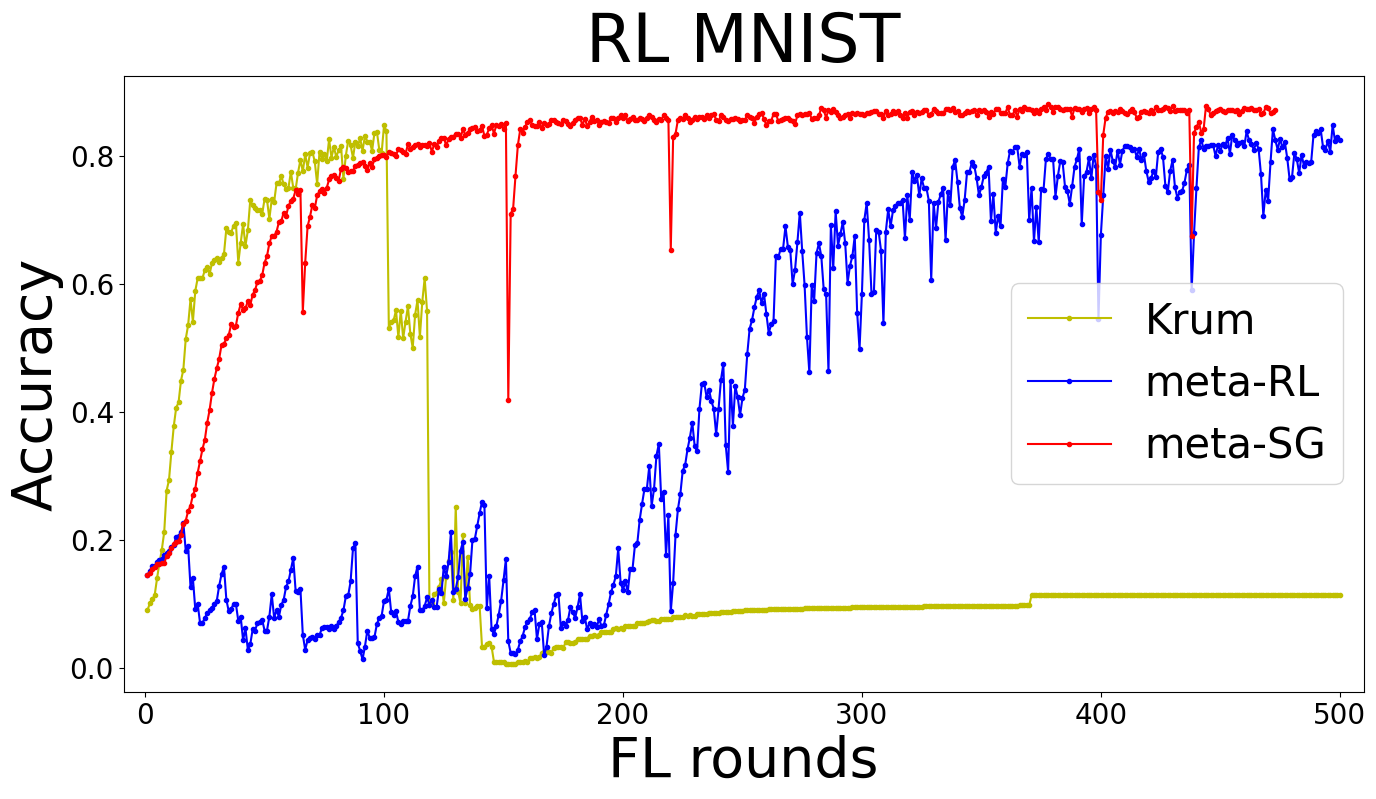}
          \caption{}
  \end{subfigure}
  \hfill
    \begin{subfigure}{0.24\textwidth}
      \centering
          \includegraphics[width=\textwidth]{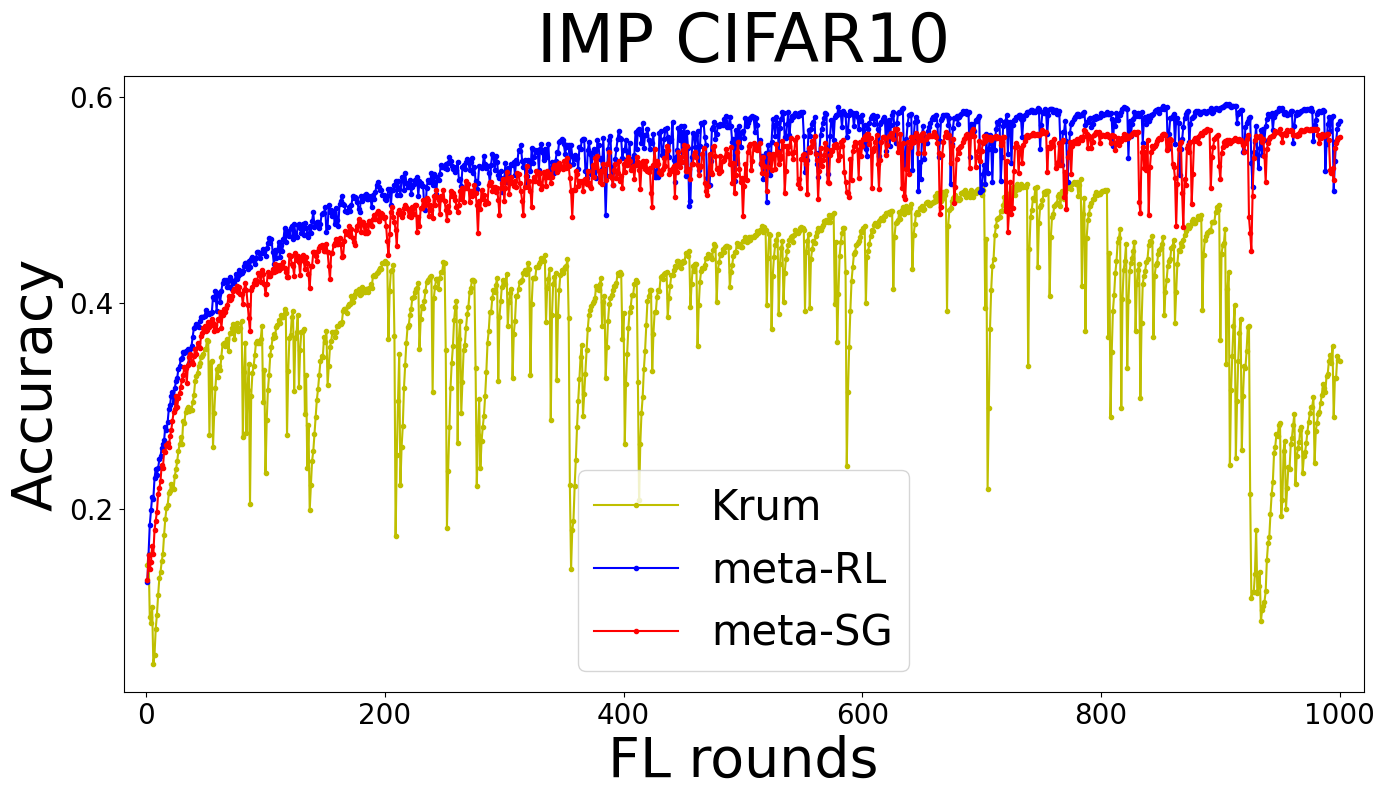}
          \caption{}
  \end{subfigure}
  \hfill
    \begin{subfigure}{0.24\textwidth}
      \centering
          \includegraphics[width=\textwidth]{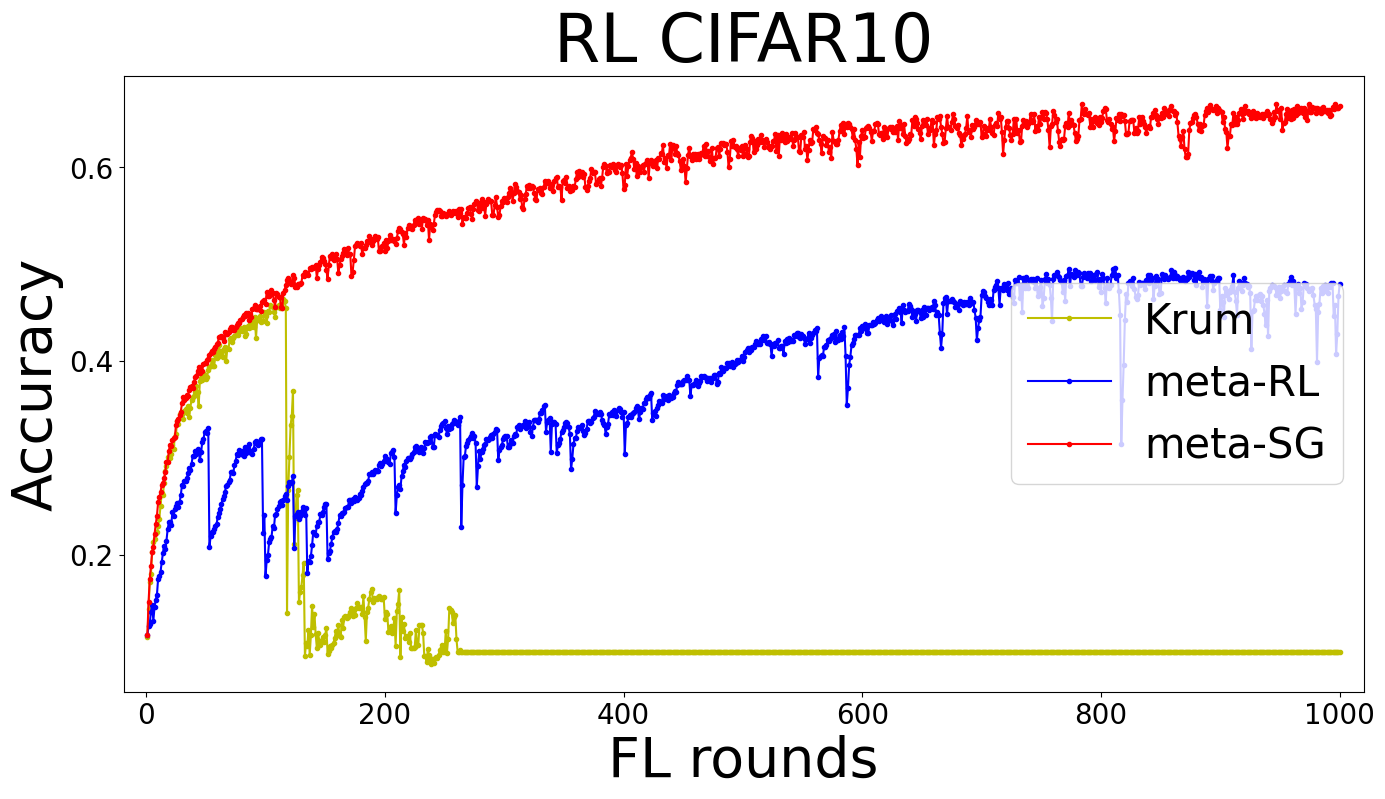}
          \caption{}
  \end{subfigure}
  \caption{\small{Comparisons of defenses against untargeted model poisoning attacks (i.e., IPM and RL) on MNIST and CIFAR-10. RL-based attacks are trained before epoch 0 against the associate defenses (i.e., Krum and meta-policy of meta-RL/meta-SG). All parameters are set as default.}}\label{fig:additional}
\end{figure}

\textbf{Adaptive Attacker Scenario.} Our meta-SG defense exhibits superior resistance to RL-based attacks~\cite{li2022learning,li2023learning}, as demonstrated in Figures~\ref{fig:meta-SG}(a), \ref{fig:untargeted}(b), \ref{fig:untargeted}(d), and \ref{fig:ablation}(a) and \ref{fig:ablation}(b). At the onset of the meta-SG's pre-training phase, we incorporate RL-based initiating policies into the attack set. This set includes RL confronting Clipping Median~\cite{li2022learning} and Krum~\cite{blanchard2017machine} for untargeted model poisoning, along with BRL countering Norm Bounding~\cite{sun2019can} and Neuron Clipping~\cite{wang2022universal} (or Pruning~\cite{wu2020mitigating}) for backdoor attacks.
Through alternating interactions with adaptive attacks, our meta-SG defenses exhibit markedly superior performance compared to fixed parameter defenses (such as Clipping Median, Krum, and FLTrust), as well as meta-RL, which solely interacts with non-adaptive attacks, namely IPM, LMP, and EB~\cite{bhagoji2019analyzing}. As displayed in Figures~\ref{fig:untargeted}(b) and~\ref{fig:untargeted}(d), the effectiveness of meta-RL~\cite{li2023robust} steadily enhances owing to online adaptation. On the other hand, meta-SG's performance quickly escalates and consistently sustains a high level (around {90\%} for MNIST and roughly {70\%} for CIFAR-10).

\textbf{Uncertain/Unknown Attacker Type}. Our meta-SG (or meta-RL) strategy possesses the capability to promptly adapt to an RL-based (or non-adaptive) attack present within the training attack set. Additionally, our approach exhibits commendable robustness against unexpected non-adaptive attacks by learning from a variety of worst-case conditions (i.e., RL-based attacks), as depicted in Figure~\ref{fig:meta-SG}(b).
However, as pointed out in Figures~\ref{fig:untargeted}(a) and \ref{fig:untargeted}(c), the defensive effectiveness of meta-SG when confronted with an unidentified attacker type is slightly inferior compared to the performance of meta-RL under uncertain conditions. This difference is less than {5\%} when the model approaches convergence.

\textbf{More Untargetd Model Poisoning Results.} Similar to results in Figure~\ref{fig:untargeted}, meta-RL achieves the best performance (slightly better than meta-SG) under IPM attacks for both MNIST and CIFAR-10. On the other hand, meta-SG performs the best (significantly better than meta-RL) against RL-based attacks for both MNIST and CIFAR-10. Notably, Krum can be easily compromised by RL-based attacks by a large margin.  In contrast, meta-RL gradually adapts to adaptive attacks, while meta-SG displays near-immunity against RL-based attacks.

\textbf{Backdoor Defense Settings.} 
In our experiments, we test two settings when the server knows the backdoor trigger but is uncertain about the target label and when the server knows the target label but not the backdoor trigger. In the former case,  we generate 10 GAN-based models~\citep{doan2021lira} targeting all 10 classes in CIFAR-10 in the simulated environment to train a defense policy in a \textbf{blackbox} setting, and 5 GAN-based models targeting classes 0-4 in the simulated environment to train a defense policy in a \textbf{graybox} setting, respectively. We then apply another GAN-based model targeting class 0 (airplane) to test the defense in each setting, with results shown in Figure~\ref{fig:ablation}(a) and Figure~\ref{fig:backdoors}(c).  In the latter case where the defender does not know the true backdoor trigger used by the attacker, we also implement GAN-based models~\citep{doan2021lira} to generate distributions of triggers (see Figure~\ref{fig:gan_trigger}) targeting one known label (truck) to simulate a black-box setting, as well as using 4 fixed sub-triggers (see Figure~\ref{fig:cifar10_dba}) targeting on one known label (truck) to simulate a \textbf{graybox} setting, and train a defense policy for each setting, and then apply a fixed global pattern (see Figure~\ref{fig:cifar10_dba}) in the real FL environment to test the defense (results shown in Figure~\ref{fig:ablation}(b) and Figure~\ref{fig:backdoors}(d)). In the default \textbf{whitebox} setting, the server knows the backdoor trigger pattern (global) and the targeted label (truck). Results in this setting are shown in Figures~\ref{fig:backdoors}(a) and~\ref{fig:backdoors}(b).

\textbf{Results on Backdoor Defenses.}
Post-stage defenses alone, such as Neuron Clipping and Pruning, are susceptible to RL-based attacks once the defense mechanism is known. However, as depicted in Figure~\ref{fig:backdoors}(a) and (b), we demonstrate that our whitebox meta-SG approach is capable of effectively eliminating the backdoor influence while preserving high main task accuracy simultaneously. 
Figure~\ref{fig:backdoors}(c) illustrates that graybox meta-SG exhibits a more stable and robust mitigation of the backdoor attack compared to blackbox meta-SG. Furthermore, in  Figure~\ref{fig:backdoors}(d), graybox meta-SG demonstrates a significant reduction in the impact of the backdoor attack, achieving nearly a $70\%$ mitigation, outperforming blackbox meta-SG.

We examine a blackbox backdoor attack scenario where the backdoor trigger or the targeted label is either unknown or uncertain to the server. The meta-policy of our meta-SG is designed to counter ten GAN-based RL attacks, each targeting a different label out of the ten possible labels in CIFAR-10 (see our technical report~\cite{li2023first} for details). Figure~\ref{fig:ablation}(a) portrays the defensive performance of this meta-policy against a GAN-based assault aimed at label $0$ within the genuine FL framework. While meta-SG considerably diminishes the backdoor accuracy (by nearly two-thirds), its performance fluctuates due to occasional instances where the meta-policy incorrectly identifies a target label, even with adaptation.
In Figure~\ref{fig:ablation}(b), we see a case where the meta-policy, having been trained with GAN-based BRL attacks, is pitted against a BRL attack utilizing a fixed global pattern~\cite{li2023learning}. While meta-SG demonstrates a significantly better defense against backdoor attacks compared to the baseline, backdoor accuracy still nears about {50\%} at the end of FL training, which is higher than in a whitebox case (approximately {10\%}), where the defender is aware of both the backdoor trigger and target label.

\textbf{Malicious Clients Number and Non-i.i.d. Degree.} We employ our meta-SG model in this context to examine the effect of imperfect knowledge about the number of malicious clients and the extent of non-i.i.d. distribution in clients' local data. We make an approximate assumption that the number of malicious clients falls within the range of $5$-$60$. During the pre-training stage, we consider only three conditions, where there are $40$, $50$, and $60$ malicious clients, respectively, as part of the potential attack scenarios. In the same way, we assume that the non-i.i.d. degree lies between $0.1$ and $0.7$, and we include non-i.i.d. levels of $0.5$, $0.6$, and $0.7$ during the pre-training phase. As illustrated in Figures~\ref{fig:ablation}(c) and~\ref{fig:ablation}(d), meta-SG delivers the highest model accuracy regardless of the number of malicious clients and non-i.i.d. levels under LMP, when the attack type is recognized by the defender.

\begin{figure}[t]
    \centering
  \begin{subfigure}{0.24\textwidth}
      \centering
          \includegraphics[width=\textwidth]{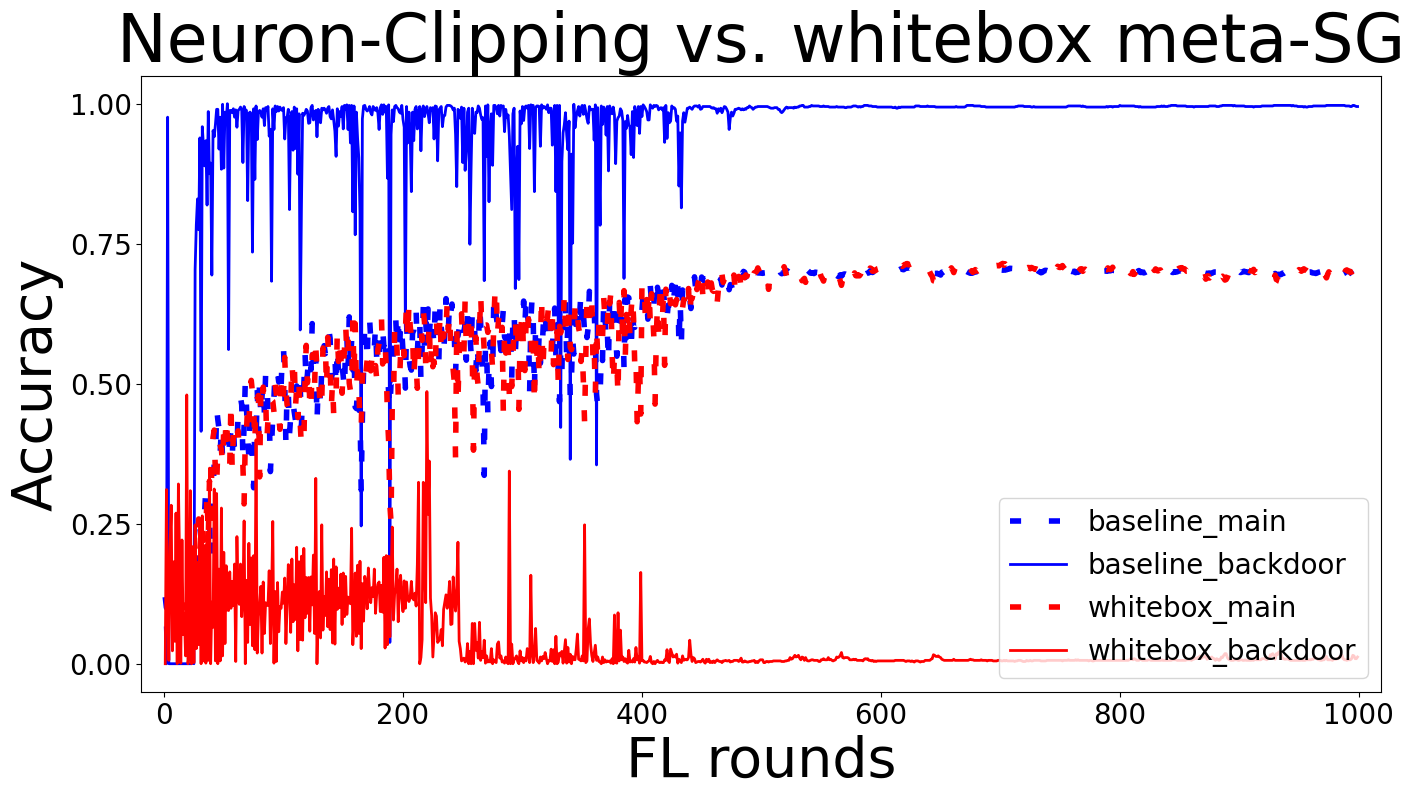}
          \caption{}
  \end{subfigure}
  \hfill
    \begin{subfigure}{0.24\textwidth}
      \centering
          \includegraphics[width=\textwidth]{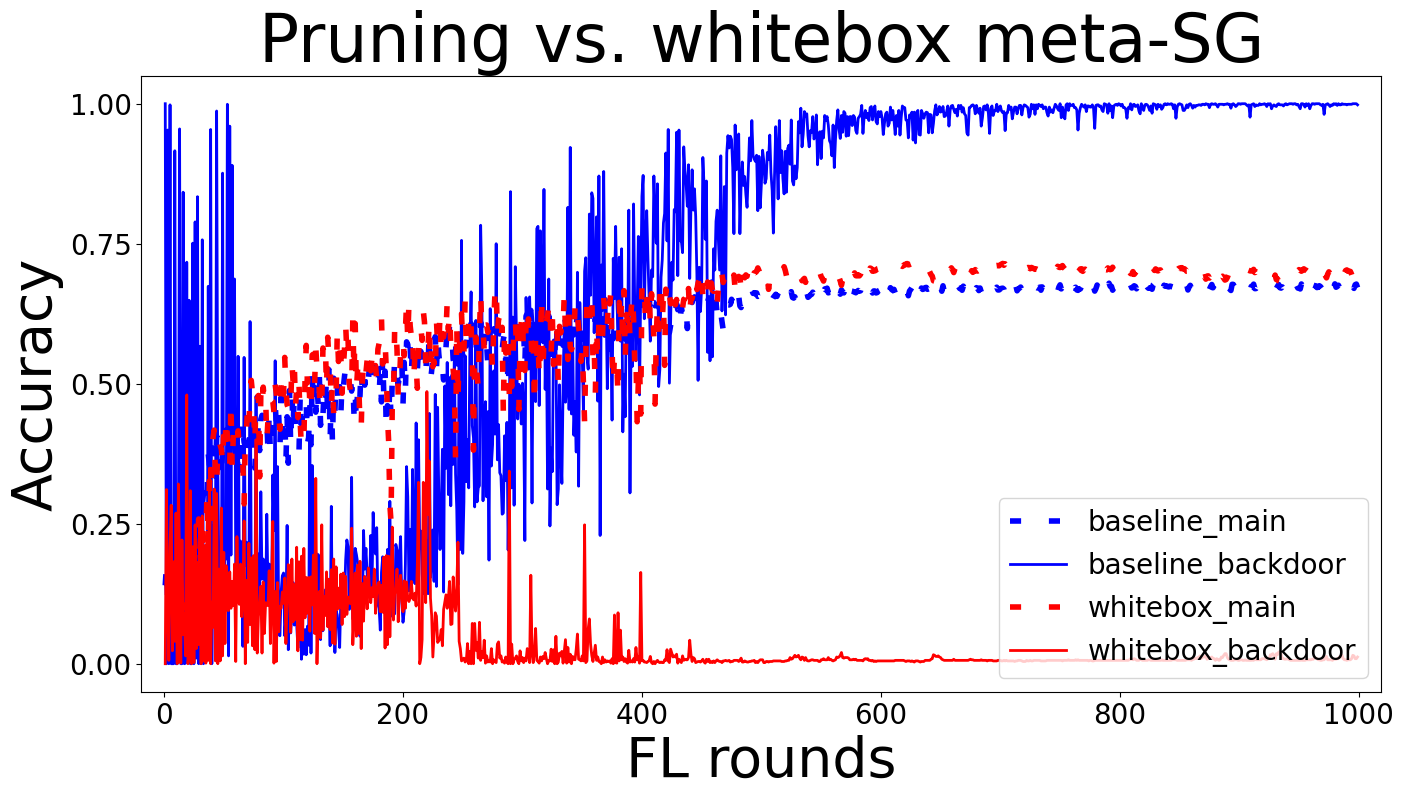}
          \caption{}
  \end{subfigure}
  \hfill
    \begin{subfigure}{0.24\textwidth}
      \centering
          \includegraphics[width=\textwidth]{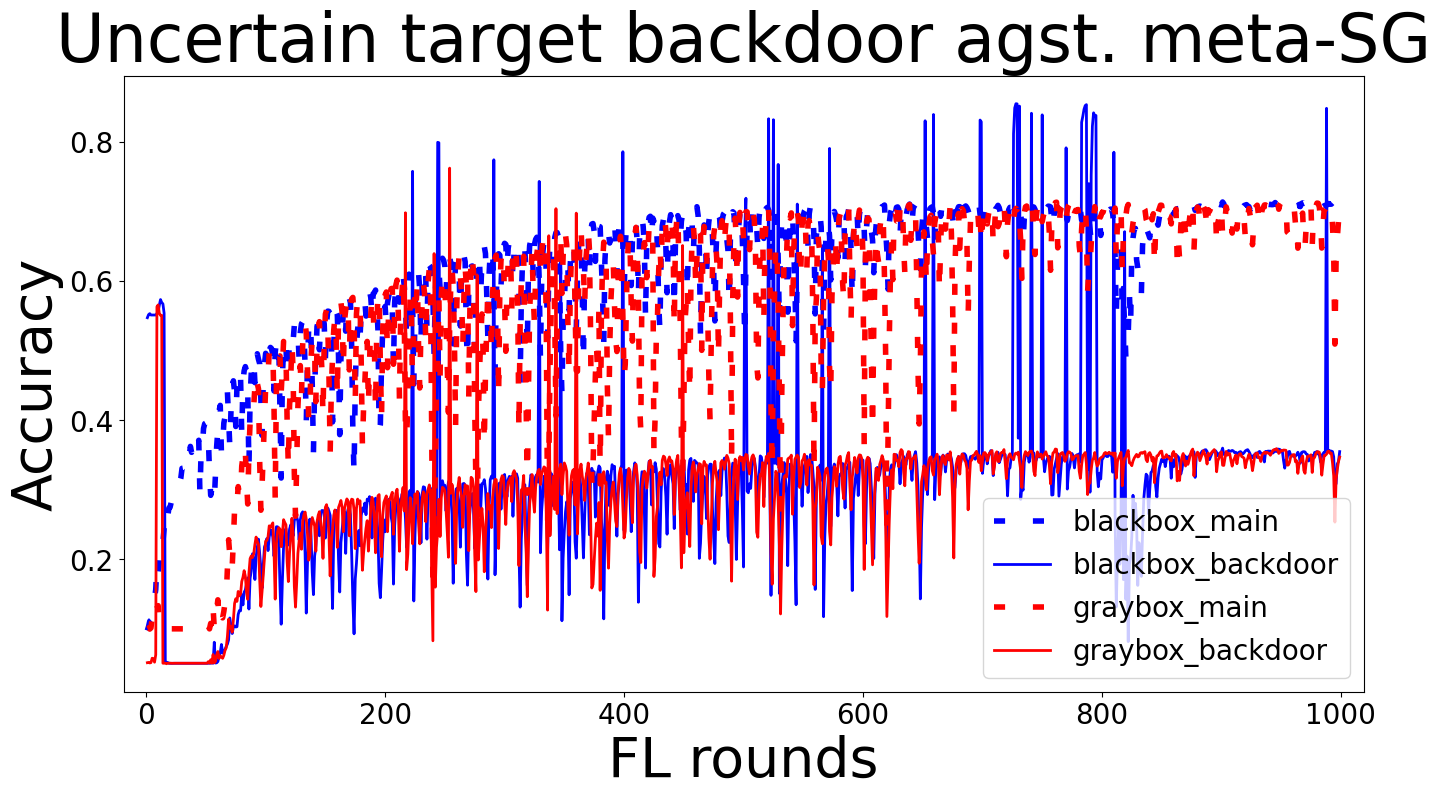}
          \caption{}
  \end{subfigure}
  \hfill
    \begin{subfigure}{0.24\textwidth}
      \centering
          \includegraphics[width=\textwidth]{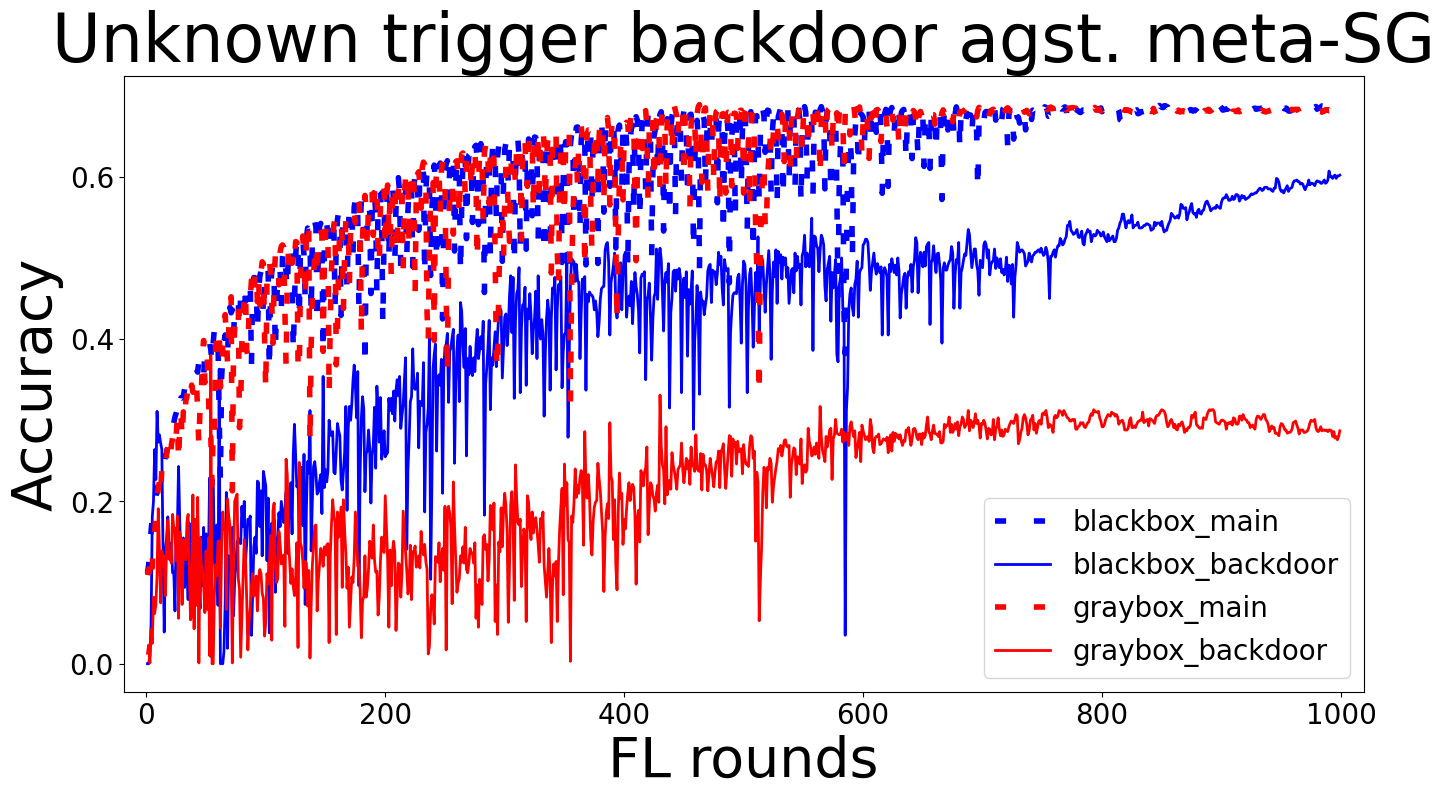}
          \caption{}
  \end{subfigure}
  \caption{\small{Comparisons of defenses (i.e., Neuron Clipping, Pruning, and meta-SG) under RL-based backdoor attack (BRL) on CIFAR-10. The BRLs are trained before epoch 0 against the associate defenses (i.e., Neuron Clipping, Pruning, and meta-policy of meta-SG). 
  Other parameters are set as default}}\label{fig:backdoors}
\end{figure}

\begin{figure*}[t]
 	\vspace{-5pt}
 	\centering
 		\subfloat[]{%
 		\includegraphics[width=0.25\textwidth]{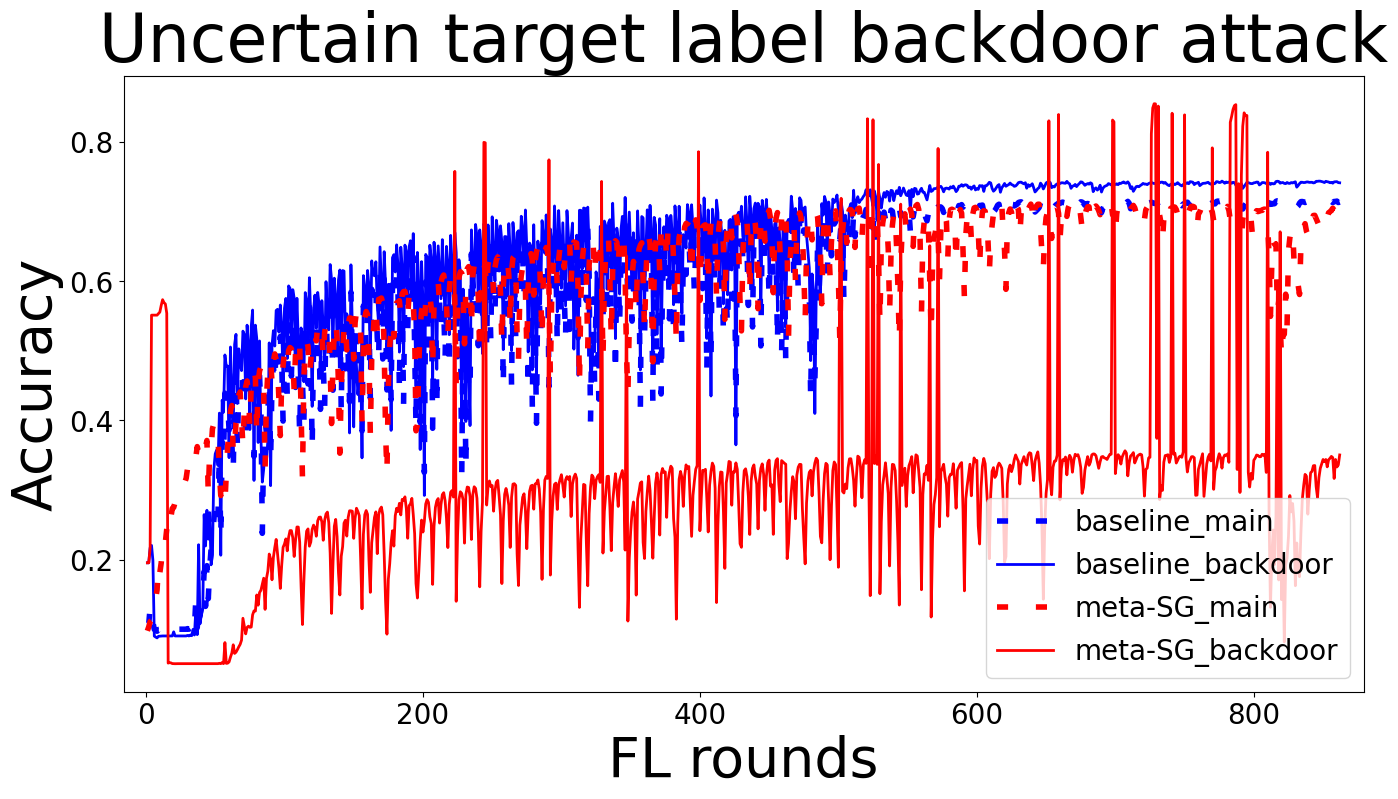}%
 	}
 	\subfloat[]{%
 		\includegraphics[width=0.2575\textwidth]{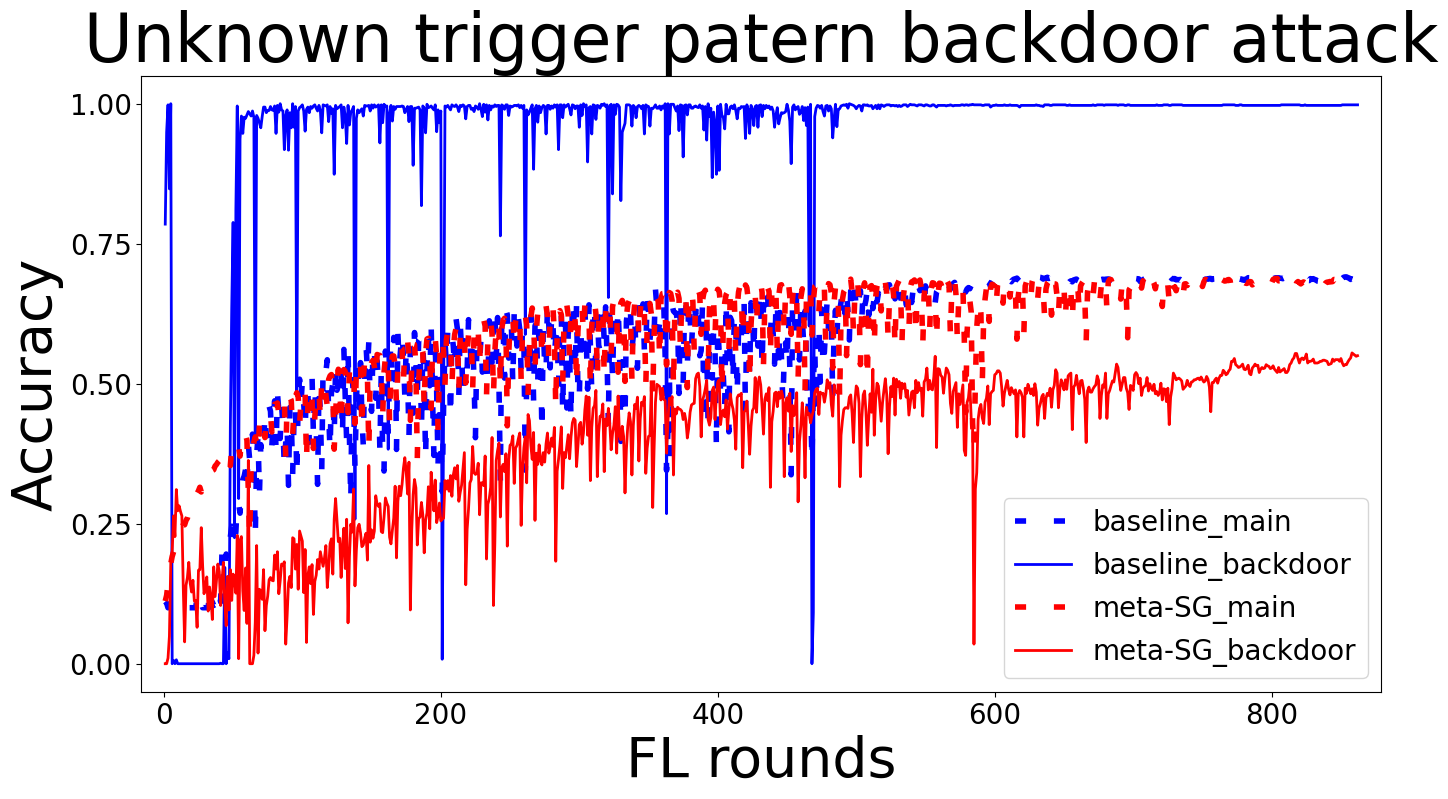}%
 	}
 	\subfloat[]{%
 		\includegraphics[width=0.24\textwidth]{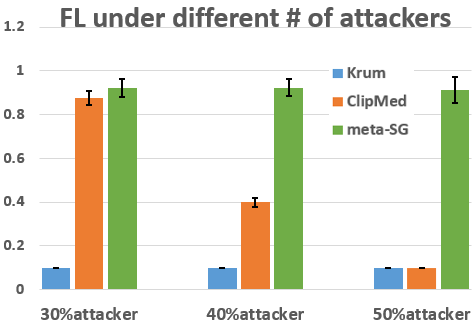}%
 	}
 	\subfloat[]{%
 		\includegraphics[width=0.2475\textwidth]{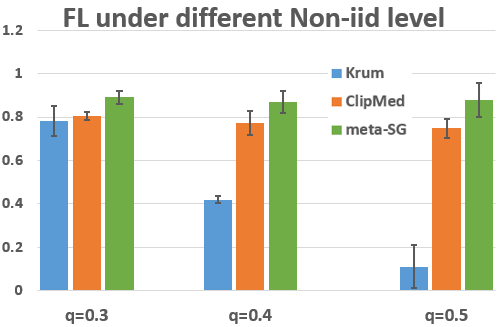}%
 	}

 	 \caption{\small Ablation Studies. For (a)-(b): uncertain backdoor            targets and unknown backdoor triggers, where the meta-policies are          trained on trigger distributions generated by GAN-based models              targeting multiple labels on CIFAR-10. For a baseline defense, the          combination of a training-stage norm bounding with a boundary of $2$        and post-training Neuron Clipping with a clipping range of $7$ is           used. For (c)-(d): meta-SG trained with the number of malicious             clients in the range of 40 to 60 and a non-$i.i.d.$ level of q= [0.5,       0.6, 0.7] on MNIST. This is compared with the defenses of Krum and          Clipping Median under the known LMP attack. Other parameters set as         default.}
 	\label{fig:ablation}
 \end{figure*}
}

\section{Related Works}

\paragraph{Poisoning/backdoor attacks and defenses in FL.} Various methods for compromising the integrity of a federated learning target model have been introduced, including targeted poisoning attacks which strive to misclassify a particular group of inputs, as explored in the studies by~\citep{bhagoji2019analyzing, baruch2019little}. Other techniques, such as those studied by~\citep{fang2020local, xie2020fall,shejwalkar2021manipulating}, focus on untargeted attacks with the aim of diminishing the overall model accuracy. The majority of existing strategies often utilize heuristics-based methods (e.g.,~\citep{xie2020fall}), or they focus on achieving a short-sighted goal (~\citep{fang2020local,shejwalkar2021manipulating}). On the other hand, malicious participants can easily embed backdoors into the aggregated model while maintaining the model’s performance on
the main task with model replacement~\cite{bagdasaryan2020backdoor}. To enhance the surreptitious nature of these poisoned updates, triggers can be distributed across multiple cooperative malicious devices, as discussed by Xie et al. (2019)\citep{xie2019dba}, and edge-case backdoors can be employed, as demonstrated by Wang et al. (2020)~\citep{wang2020attack}. However, these methods can be sub-optimal, especially when there's a need to adopt a robust aggregation rule. Additionally, these traditional methods typically demand access to the local updates of benign agents or precise parameters of the global model for the upcoming round~\citep{xie2020fall,fang2020local} in order to enact a significant attack. In contrast to these methods, RL-based approach~\cite{li2022learning, shen2021coordinated, li2023learning} employs reinforcement learning for the attack, reducing the need for extensive global knowledge while focusing on a long-term attack goal.

Several defensive strategies have been suggested to counter model poisoning attacks, which broadly fall into two categories: those based on robust aggregation and those centered around detection.
Robust-aggregation-based defenses encompass techniques such as dimension-wise filtering. These methods treat each dimension of local updates individually, as explored in studies by~\citep{bernstein2018signsgd,yin2018byzantine}. Another strategy is client-wise filtering, the goal of which is to limit or entirely eliminate the influence of clients who might harbor malicious intent. This approach has been examined in the works of~\citep{blanchard2017machine,pillutla2022robust,sun2019can}. Some defensive methods necessitate the server having access to a minimal amount of root data, as detailed in the study by~\cite{cao2020fltrust}. 
Naive backdoor attacks are limited by even simple defenses like norm-bounding ~\citep{sun2019can} and weak differential private ~\citep{geyer2017differentially} defenses.
Despite to the sophisticated design of state-of-the-art non-addaptive backdoor attacks against federated learning, post-training stage defenses ~\citep{wu2020mitigating,nguyen2021flame,rieger2022deepsight} can still effectively erase suspicious neurons/parameters in the backdoored model.


\paragraph{Multi-agent meta learning. } Meta-learning, and in particular meta-reinforcement-learning aim to create a generalizable policy that can fast adapt to new tasks by exploiting knowledge obtained from past tasks \cite{duan2016rl,finn2017model}. The early use cases of meta-learning have been primarily single-agent tasks, such as few-shot classification and single-agent RL \cite{finn2017model}. A recent research thrust is to extend the meta-learning idea to multi-agent systems (MAS), which can be further categorized into two main directions: 1) distributed meta-learning in MAS \cite{kayaalp2022dif,zhang2022distributed}; 2) meta-learning for generalizable equilibrium-seeking \cite{gupta2021dynamic,harris2022meta,zhao2022stackelberg,tao23ztd}.  The former focuses on a decentralized operation of meta-learning over networked computation units to reduce computation/storage expenses. The latter is better aligned with the original motivation of meta-learning, which considers how to solve a new game (or multi-agent decision-making) efficiently by reusing past experiences from similar occasions.            

In stark contrast to the existing research efforts, our work leverages the adaptability of meta-learning to address information asymmetry in dynamic games of incomplete information, leading to a new equilibrium concept: meta-equilibrium (see \Cref{def:meta-se}). What distinguishes our work from the aforementioned ones is that 1) every entity in our meta-SG is a self-interest player acting rationally without any coordination protocol; 2) meta-learning in our work is beyond a mere solver for computing long-established equilibria (e.g., Stackelberg equilibrium); it brings up a non-Bayesian approach to processing information in dynamic games (see \Cref{app:meta-se}), which is computationally more tractable. This meta-equilibrium notion has been proven effective in combating information asymmetry in adversarial FL. Since asymmetric information is prevalent in security studies, our work can shed light on other related problems. 

\paragraph{First-order methods in bilevel optimization.} The meta-SG problem in \eqref{eq:meta-se} amounts to a stochastic bilevel optimization. The meta-SL in \Cref{algo:meta-sl} admits a much simpler gradient estimation than what one would often observe in the bilevel optimization literature \cite{chen23bilevel, kwon23fully}, where the gradient estimate for the upper-level problem involves a Hessian inverse \cite{chen23bilevel} or some first-order correction terms \cite{kwon23fully}. The key intuition behind this simplicity lies in the strict competitiveness (see \Cref{ass:sc}). Informally speaking, \eqref{eq:meta-se} is more akin to minimax programming \cite{nouiehed2019solving,li2022sampling}, even though it is a general-sum game. However, the data-driven meta-adaptation within the value function in  \Cref{eq:meta-se} leads to a more involved gradient estimation. since the data induces extra randomness in addition to policy gradient estimates \cite{fallah2021convergence}. Perhaps, the closest to our work is \cite{li2022sampling} where the authors investigate adversarial meta-RL and arrive at a similar Stackelberg formulation. However, \cite{li2022sampling} considers a minimax relaxation to the original Stackelberg formulation, leading to simpler nonconvex programming. Our work is among the first endeavors to investigate fully first-order algorithms for solving general-sum Stackelberg games.

\end{document}